\documentclass{article}

\PassOptionsToPackage{numbers, compress}{natbib}
\usepackage[final]{neurips_2021}

\usepackage[utf8]{inputenc} % allow utf-8 input
\usepackage[T1]{fontenc}    % use 8-bit T1 fonts
\usepackage{hyperref}       % hyperlinks
\hypersetup{
    pdfsubject={Vector-valued Distance and Gyrocalculus on the Space of Symmetric Positive Definite Matrices},
    pdftitle={Vector-valued Distance and Gyrocalculus on the Space of Symmetric Positive Definite Matrices},
    pdfauthor={Federico Lopez},
    pdfkeywords={symmetric spaces, spd space, spd manifold, symmetric positive definite matrices, spd, riemannian manifold, rotations, reflections, translations, scaling, gyro vector, gyro calculus, gyro groups, tangent space optimization, non euclidean optimization, hyperbolic geometry, hyperbolic space, matrix models,  non-euclidean geometry, finsler metrics, finsler distance, finsler geometry, vector valued distance, vector valued distance function, deep learning, riemannian manifold learning, manifold learning, geometric deep learning, graph embeddings, knowledge graph embeddings, item recommendations, question answering}
}
\usepackage{url}            % simple URL typesetting
\usepackage{booktabs}       % professional-quality tables
\usepackage{amsfonts}       % blackboard math symbols
\usepackage{nicefrac}       % compact symbols for 1/2, etc.
\usepackage{microtype}      % microtypography
\usepackage[dvipsnames]{xcolor}

\usepackage{todonotes}
\usepackage{amsmath}
\usepackage{amssymb}
\usepackage{adjustbox}
\usepackage{wrapfig}
\usepackage{amsthm}
\usepackage{subfig}
\usepackage{multirow}   % for tables with multirow
\usepackage{wrapfig}

\newcommand{\Id}{{\rm Id}}
\newcommand{\R}{\mathbb R}

\newcommand{\tr}{{\rm tr}}

\newcommand{\bpm}{\begin{pmatrix}}
\newcommand{\epm}{\end{pmatrix}}
\newcommand{\bsm}{\left(\begin{smallmatrix}}
\newcommand{\esm}{\end{smallmatrix}\right)}

\newcommand{\bigO}{\mathcal{O}}
\newcommand{\HH}{\mathbb H}
\newcommand{\RR}{\mathbb R}

\newcommand{\spd}{\operatorname{SPD}}
\newcommand{\diag}{\operatorname{diag}}

\newcommand{\SPDtraRiem}{$\operatorname{SPD}^{R}_{\rm{Sca}}$}
\newcommand{\SPDtraFone}{$\operatorname{SPD}^{F_1}_{\rm{Sca}}$}
\newcommand{\SPDrotRiem}{$\operatorname{SPD}^{R}_{\rm{Rot}}$}
\newcommand{\SPDrotFone}{$\operatorname{SPD}^{F_1}_{\rm{Rot}}$}
\newcommand{\SPDrefRiem}{$\operatorname{SPD}^{R}_{\rm{Ref}}$}
\newcommand{\SPDrefFone}{$\operatorname{SPD}^{F_1}_{\rm{Ref}}$}

\newtheorem{proposition}{Proposition}
\newtheorem{lemma}{Lemma}
\newtheorem{corollary}{Corollary}

\newcommand{\EntitySub}{\mathcal{E}}
\newcommand{\RelSub}{\mathcal{R}}
\newcommand{\TripSub}{\mathcal{T}}

\newcommand{\gyr}{\operatorname{gyr}}
\newcommand{\Aut}{\operatorname{Aut}}

\title{Vector-valued Distance and Gyrocalculus on the Space of Symmetric Positive Definite Matrices}

\author{%
   Federico L\'opez\thanks{Correspondence to \texttt{federico.lopez@h-its.org}} \\
   HITS \\
   \And
   Beatrice Pozzetti \\
   Heidelberg University \\
   \And
   Steve Trettel \\
   Stanford University \\
   \AND
   Michael Strube \\
   HITS \\
   \And
   Anna Wienhard \\
   Heidelberg University \\
   HITS\\
}

\begin{document}

\maketitle

\begin{abstract}
We propose the use of the vector-valued distance to compute distances and extract geometric information from the manifold of symmetric positive definite matrices (SPD), and develop gyrovector calculus, constructing analogs of vector space operations in this curved space.
We implement these operations and showcase their versatility in the tasks of knowledge graph completion, item recommendation, and question answering. In experiments, the SPD models outperform their equivalents in Euclidean and hyperbolic space.
The vector-valued distance allows us to visualize embeddings, showing that the models learn to disentangle representations of positive samples from negative ones.
\end{abstract}

\section{Introduction}
%\todo{MS: somehow we do not mention that we are interested in learning (graph) embeddings}
%% Introduction of the problem we want to solve: Why?
% spds are good for this properties: A, B and C, therefore, they have been applied in X, Y and Z.
Symmetric Positive Definite (SPD) matrices have been applied in many tasks in computer vision such as pedestrian detection \cite{tosato2010multiclassClassifOnRiemManifolds,tuzel2008spdPedestrianDetection}, action \cite{harandi2014manToManforActionRecog,liandlu2018spdLieGroupAlgorithm,nguyen2019spdHandGestureRecog} or face recognition \cite{huang2014euclidToRiemPointClassif, huang2015logEuclideanSPD}, object \cite{ionescu2015matrixBackprop, Yin2016KernelClusteringSPD} and image set classification \cite{wang2018analysisOnRiemManifoldImageSets}, visual tracking \cite{wu2015manifoldKernelforVisualTracking}, and medical imaging analysis \cite{arsigny2006spdGeometricMeans,pennec2006spdmagneticResonanceImaging} among others.
They have been used to %model relationships of spatial or temporal domain (graph Laplacians \cite{liandlu2018spdLieGroupAlgorithm} or to 
capture statistical notions (Gaussian distributions \cite{salem2017gaussianDistribOnSPD}, covariance \cite{tuzel2006regionCovariance}), while respecting the Riemannian geometry of the underlying SPD manifold, which offers a convenient trade-off between structural richness and computational tractability \cite{cruceru20matrixGraph}.
%% What previous work has made to solve the problem and its issues
%The adoption of neural networks and deep learning in these non-Euclidean settings has been rather limited until very recently [Not sure I want to say this since there are quite a few attempts and we do not compare to them], the main reason being the non-trivial or impossible principled general- izations of basic operations (e.g. vector addition, matrix-vector multiplication, vector translation, vector inner product) as well as, in more complex geometries, 
% problems with metrics:
%Moreover, since Euclidean distance is no longer a suitable metric on SPD manifolds, p
Previous work has applied approximation methods that locally flatten the manifold by projecting it to its tangent space \cite{carreira2012semanticSegmentationSPD, Vemulapalli2015RiemannianML}, or by embedding the manifold into higher dimensional Hilbert spaces \cite{quang2014logHilbertSPD, Yin2016KernelClusteringSPD}. 

These methods face problems such as distortion of the geometrical structure of the manifold and other known concerns with regard to high-dimensional spaces \cite{dong2017spdToFaceRecog}.
To overcome these issues, several distances on SPD manifolds have been proposed, such as the Affine Invariant metric \cite{pennec2006spdmagneticResonanceImaging}, the Stein metric \cite{sra2012steinMetricSPD}, the Bures–Wasserstein metric \cite{bhatia2019buresWassersteinDistanceSPD} or the Log-Euclidean metric \cite{arsigny2006spdGeometricMeans, arsigny2006logEuclidTensorDiffusion}, with their respective geometric properties. 
%todo{F: why? this is a big claim/critic that we should be able to justify}{However, none of these have used the full representation power of SPD.} 
However, %\todo{Provide some example of this? 1) Previous work distorts the space with the aim of using euclidean tools. 2) The VVD can distinguish more things than the scalar distance}
{the representational power of SPD is not fully exploited} in many cases \cite{arsigny2006logEuclidTensorDiffusion,pennec2006spdmagneticResonanceImaging}. At the same time, it is hard to translate operations into their non-Euclidean domain given the lack of closed-form expressions. %, classic tools do not have a correspondence in these geometries.
There has been a growing need to generalize basic operations, such as addition, rotation, reflection or scalar multiplication, to their Riemannian geometric counterparts to leverage this structure in the context of Geometric Deep Learning \cite{bronstein2018geomdeeplearning}.

%% what to we do to solve these issues: here we sell the paper without explaining how we do all this things but just the advantages
%Whereas the relative position between two points in Euclidean space up to rigid motions is captured completely by their distance, and the distance is the only invariant, invariant of two points in SPD (up to isometry) is given by a vector. 

% In Euclidean or hyperbolic space the relative position between two points is completely captured by a single number (their relative distance).
% In contrast, the richer geometry of SPD provides an n-dimensional invariant of pairs of points: the vector valued distance.

SPD manifolds have a rich geometry that contains both Euclidean and hyperbolic subspaces. Thus, embeddings into SPD manifolds are beneficial, since they can accommodate hierarchical structure in data sets in the hyperbolic subspaces while at the same time represent Euclidean features. This makes them more versatile than using only hyperbolic or Euclidean spaces, and in fact, their different submanifold geometries can be used to identify and disentangle such substructures in graphs.

In this work, we introduce the vector-valued distance function to exploit the full representation power of SPD (\S\ref{sec:VVD}). While in Euclidean or hyperbolic space the relative position between two points is completely captured by their distance (and this is the only invariant), in SPD this invariant is a vector, encoding much more information  than a single scalar. This vector reflects the higher expressivity of SPD due to its richer geometry encompassing Euclidean as well as hyperbolic spaces. %, which it both contains.
We develop algorithms using the vector-valued distance and showcase two main advantages: its versatility to implement universal models,  %\todo{F: I wouldn't stress this much since we do not show any model like this}{in which the best metric can be learned}, 
and its use in explaining and visualizing what the model has learned. 

%In this work, we develop a unified framework that enables us to apply different Riemannian and Finsler metrics integrated with a Riemannian optimization scheme. Through this framework, we can systematically alternate between distance functions in SPD manifolds, exploiting different structural advantages of the space.
Furthermore, we bridge the gap between Euclidean and SPD geometry by developing gyrocalculus in SPD (\S\ref{sec:gyrocalculus}), which yields closed-form expressions of arithmetic operations, such as addition, scalar multiplication and matrix scaling.
%Our methods faithfully respect the geometric structure of the Riemannian manifold and thus benefit from its enhanced representation capabilities.
% how do we do it
%By leveraging the rich symmetric structure of SPD manifolds, we generalize notions of gyrocalculus, giving rise to geometrically meaningful formulas and distance computations in the space (\S\ref{sec:gyrocalculus}). 
This provides means to translate previously implemented ideas in different metric spaces to their analog notions in SPD. These arithmetic operations are also useful to adapt neural network architectures to SPD manifolds.

We showcase this on knowledge graph completion, item recommendation, and question answering. In the experiments, the proposed SPD models outperform their equivalents in Euclidean and hyperbolic space (\S\ref{sec:experiments}).
These results reflect the superior expressivity of SPD, and show the versatility of the approach and ease of integration with downstream tasks.

% talk about other advantages of vvd: explainability
%\todo{F: Integrate this with Paragraph 3 about VVD}{As} part of our metric learning framework, we introduce a vector-valued distance function on SPD manifolds (\S\ref{sec:space-spd}). 

The vector-valued distance  allows us to develop a new tool for the analysis of the structural properties of the learned representations. With this tool, we visualize high-dimensional SPD embeddings, providing better explainability on what the models learn (\S\ref{sec:analysis}). We show that the knowledge graph models are capable of disentangling and clustering positive triples from negative ones.

\section{Related Work}
%Our work explores more flexible spaces of non-constant sectional curvature/fits in the area of Riemannian manifold learning \cite{bronstein2018geomdeeplearning} for knowledge graph embeddings. It is in line with previous work in hyperbolic \cite{balazevic2019murp, chami2020lowdimkge, kolyvakis2020hyperKG} and spherical \cite{xiao2016manifoldEforKG} spaces, the torus \cite{ebisu2018toruse}, or Cartesian products of spaces \cite{han2020dyernie}, which employ non-Euclidean geometries due to their appealing properties that allow methods to represent data arising in several domains \cite{rubindelanchy2020manifold}.

% talk about previous work in SPDs
Symmetric positive definite matrices are not new in the Machine Learning literature. They have been used in a plethora of applications  \cite{arsigny2006spdGeometricMeans, dong2017spdToFaceRecog, harandi2014manToManforActionRecog, huang2017riemannianNetForSPDMatrix, huang2014euclidToRiemPointClassif,   huang2015logEuclideanSPD, ionescu2015matrixBackprop, liandlu2018spdLieGroupAlgorithm, nguyen2019spdHandGestureRecog, pennec2006spdmagneticResonanceImaging, salem2017gaussianDistribOnSPD, tosato2010multiclassClassifOnRiemManifolds, tuzel2006regionCovariance, tuzel2008spdPedestrianDetection,  wang2018analysisOnRiemManifoldImageSets, wu2015manifoldKernelforVisualTracking, Yin2016KernelClusteringSPD}), although not always respecting the intrinsic structure or the positive definiteness constraint \cite{carreira2012semanticSegmentationSPD, feragen2015curvatureAndLinearity, quang2014logHilbertSPD, Vemulapalli2015RiemannianML, Yin2016KernelClusteringSPD}.
The alternative has been to map manifold points onto a tangent space and employ Euclidean-based tools. Unfortunately, this mapping distorts the metric structure in regions far
from the origin of the tangent space affecting the performance \cite{jayasumana2013kernelOnSPD, zhao2019convexClassifSPD}. 
%To the best of our knowledge, we are the first work to apply the manifold of symmetric positive definite matrices for Knowledge graph embeddings.

% talk about previous attempts to proposed basic SPD building blocks
Previous work has proposed alternatives to the basic neural building blocks respecting the geometry of the space. For example, transformation layers \cite{dong2017spdToFaceRecog, gao2019robustRepreSPD, huang2017riemannianNetForSPDMatrix}, alternate convolutional layers based on SPDs 
\cite{zhang2018manifoldToManifold} and Riemannian means \cite{chakraborty2020manifoldNet}, or appended after the convolution \cite{brooks2019riemannianRadarData}, recurrent models \cite{chakraborty2018recurrentSPD}, projections onto Euclidean spaces \cite{li2018globalCovariance, mao2019cosonet} and batch normalization \cite{brooks2019riemBNforSPD}. Our work follows this line, providing explicit formulas for translating Euclidean arithmetic notions into SPDs.

% a comment on optimization?
Our general view, using the vector-valued distance function, allows us to treat Riemannian and Finsler metrics on SPD in a unified framework. Finsler metrics have previously been applied 
in compressed sensing \cite{Donoho:2008aa}, information geometry \cite{shen2006FinslerMetricInfoGeom}, for clustering categorical distributions \cite{Nielsen2019clusteringWithFinsler}, and in robotics \cite{ratliff2020finslerRobotics}.
With regard to optimization, matrix backpropagation techniques have been explored \cite{absil2009optimAlgosonMatrix, boumal2014manopt, ionescu2015matrixBackprop}, with some of them accounting for different Riemannian geometries \cite{brooks2019riemBNforSPD, huang2017riemannianNetForSPDMatrix}. Nonetheless, we opt for tangent space optimization \cite{chami2019hgcnn} by exploiting the explicit formulations of the exponential and logarithmic map.

\section{The Space $\spd_n$}
\label{sec:space-spd}
\begin{wrapfigure}{r}{0.5\textwidth}
\vspace{-3mm}
\centering
\includegraphics[width=.4\textwidth,keepaspectratio]{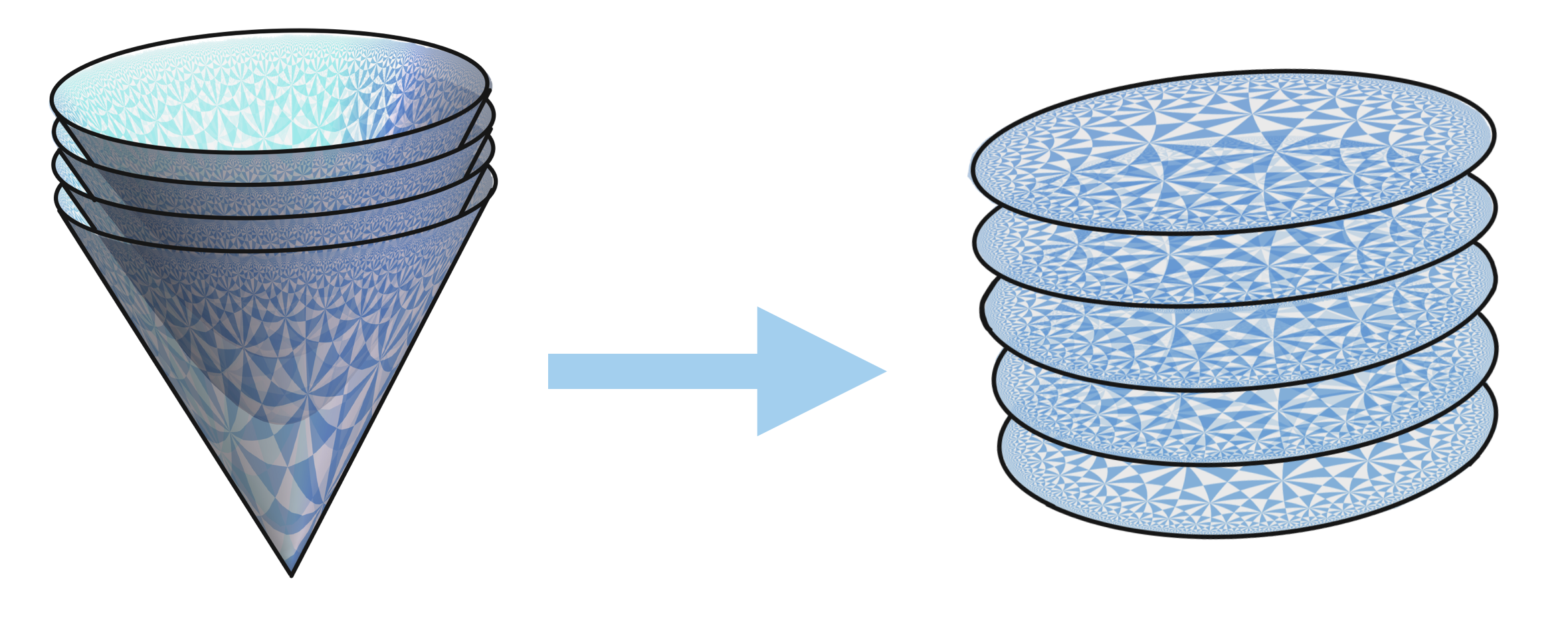}
\caption{$\spd_2$ is foliated by hyperboloids, each of which is a copy of the hyperbolic plane.}
\label{fig:spd2-cone-planes}
\vspace{-0mm}
\end{wrapfigure}
The space $\spd_n$ is a Riemannian manifold of non-positive curvature of $n(n+1)/2$ dimensions. Points in $\spd_n$ are positive definite real symmetric $n\times n$ matrices, with the identity matrix $I$ being a natural basepoint.
The tangent space to any point of $\spd_n$ can be identified with the vector space $S_n$ of all real symmetric $n\times n$ matrices.
$\spd_n$ contains $n$-dimensional Euclidean subspaces, $(n-1)$-dimensional hyperbolic subspaces as well as products of {$\lfloor\frac{n}{2}\rfloor$} hyperbolic planes (see \citet{helgason1078diffGeom} for an in-depth introduction).

In Figure~\ref{fig:spd2-cone-planes} we visualize the smallest nontrivial example. $\spd_2$ identifies with the inside of a cone in $\RR^3$, cut out by requiring both eigenvalues of the matrix $\left(\begin{smallmatrix}x&y\\y&z\end{smallmatrix}\right)$ to be positive. It carries the product geometry of the hyperbolic plane times a line.
%For an overview on the differential geometry of $\spd_n$ see Appendix \ref{sec:DiffGeo_SPD}.

\textbf{Exponential and logarithmic maps:}
The exponential map, $\exp\colon S_n\to \spd_n$, is a homeomorphism which links the Euclidean geometry of the tangent space $S_n$ and the curved geometry of $\spd_n$.
%As a consequence of the non-positive curvature, $\exp$ is a diffeomorphism of $S_n$ onto $\spd_n$, and so has an inverse: the logarithm map $\log\colon \spd_n\to S_n$.
Its inverse is the logarithm map, $\log\colon \spd_n\to S_n$. This pair of functions give diffeomorphisms that allows one to freely move between ’tangent space coordinates’ or the original ’manifold coordinates'.
We apply both maps based at $I\in \spd_n$. The reason for this is that while mathematically any two points on $\spd_n$ are equivalent, and we could obtain a different concrete expression for any other choice of basepoint $B\in\spd_n$, the resulting formulas would be more complicated, and thus $I$ is the best choice from a computational point of view. We prove this in Appendix~\ref{sec:exponential-maps}.

%\begin{figure*}[!t]
%\vspace{-2mm}
%\centering
%\includegraphics[width=.5\textwidth,keepaspectratio]{img/SPD2.png}
%\caption{The space $\spd_2$ identifies with the inside of a cone in $\RR^3$, cut out by requiring both eigenvalues of the matrix $\left(\begin{smallmatrix}x&y\\y&z\end{smallmatrix}\right)$ to be positive.  This cone is foliated by hyperboloids, each of which is an embedded copy of the hyperbolic plane.}
%\label{fig:spd2}
%\vspace{-4mm}
%\end{figure*}

\textbf{Symmetries:}
%\subsection{Symmetries}
%\todo{MS: I find the terminology confusing in this subsection: $P\mapsto MPM^T$ is a function, a symmetry, a transformation, an operation, and maybe more ...
%ST: Fixed?}
The prototypical symmetries of $\spd_n$ are parameterized by elements of $GL(n;\RR)$: any invertible matrix $M$ defines the symmetry $P\mapsto MPM^T$ acting on all points $P\in\spd_n$.
Thus many geometric transformations $\spd_n$ can be completed using standard optimized matrix algorithms as opposed to custom-built procedures. See Appendix~\ref{sec:metric_isom}. for a brief review of these symmetries.

Among these, we may find $\spd_n$-generalizations of familiar symmetries of Euclidean geometry.
When also $M$ is an element of $\spd_n$, the symmetry $P\mapsto MPM^T$ is a generalization of an Euclidean \emph{translation}, fixing no points of $\spd_n.$
When $M$ is an orthogonal matrix, the symmetry $P \mapsto MPM^{T}$ is conjugation by $M$, and thus fixes the basepoint $I=MIM^T=MM^{-1}=I$.
We think of elements fixing the basepoint as being $\spd_n$-\emph{rotations} or $\spd_n$-\emph{reflections}, when the matrix $M$ is a familiar rotation or reflection (see Figure~\ref{fig:rot-refl}).

The Euclidean symmetry of \emph{reflecting in a point} also has a natural generalization to $\spd_n$.
Euclidean reflection in the origin is given by $p\mapsto -p$; and its $\spd_n$-analog, reflection in the basepoint $I$, is matrix inversion $P\mapsto P^{-1}$.
The general $\spd_n$-reflection in a point $Q\in\spd_n$ is a conjugate of this by an $\spd_n$ translation, given by $P\mapsto Q P^{-1}Q$.

\begin{figure*}[!t]
\vspace{-2mm}
\centering
\subfloat{\includegraphics[width=.33\textwidth,keepaspectratio]{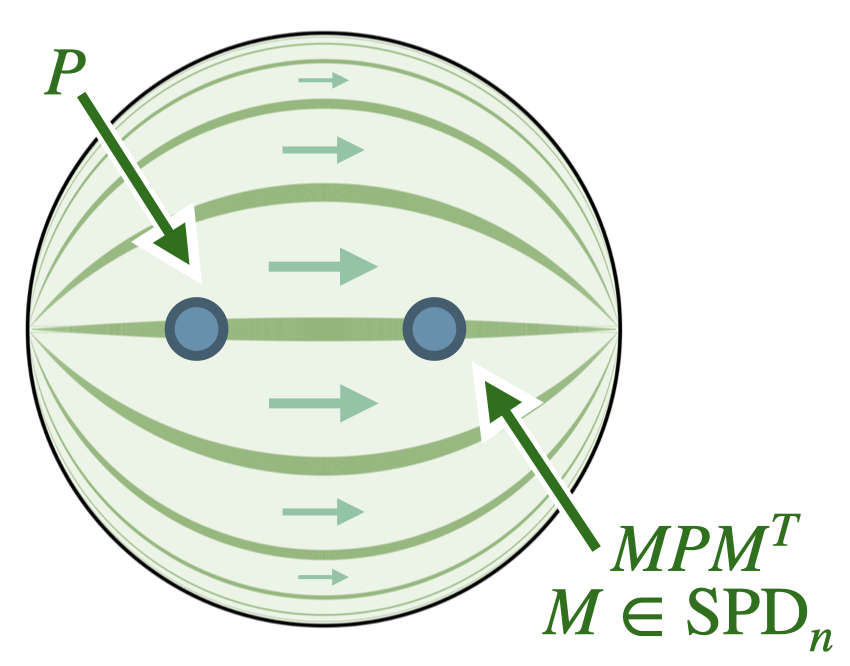}}\hfill
\subfloat{\includegraphics[width=.32\textwidth,keepaspectratio]{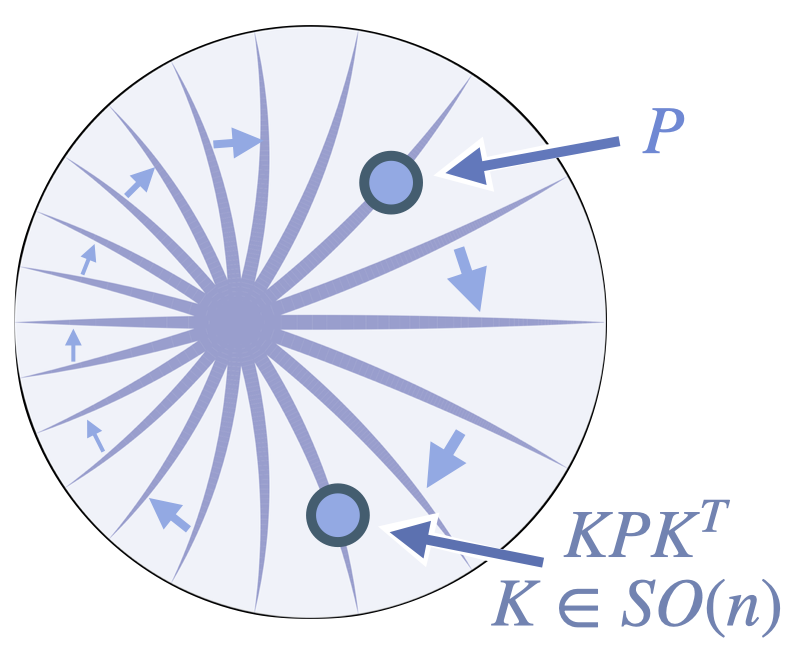}
}\hfill
\subfloat{\includegraphics[width=.33\textwidth,keepaspectratio]{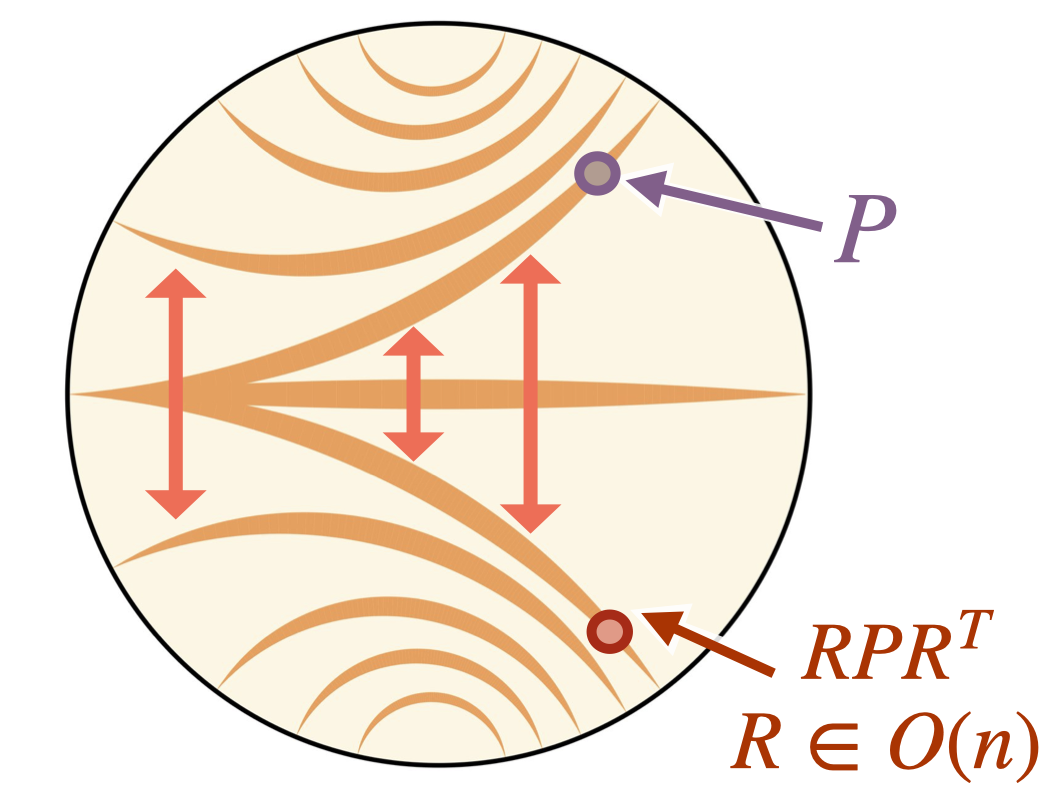}
}\hfill
\caption{Some isometries of $\spd_n$ have analogous Euclidean counterparts.  Translation (left), rotation (center) and reflection (right).}
\label{fig:rot-refl}
\vspace{-4mm}
\end{figure*}

\subsection{Vector-valued Distance Function}
\label{sec:VVD}

%The Riemannian metric on $\spd_n$ is defined as follows: 
%if $U,V\in S_n$ are tangent vectors based at $P\in \spd_n$, their inner product is:
%$$\langle U,V\rangle_P=\tr(P^{-1}UP^{-1}V).$$ It induces a distance function $d: \spd_n \times \spd_n \to \mathbb{R}$. 
In Euclidean or hyperbolic spaces, the relative position between two points %in the space 
is completely determined by their distance, which is given by a scalar. For the space $\spd_n$, it is determined by a vector. %this is not the case.
% \todo{F: why is this relevant?}{given two diagonal matrices $D_1, D_2$ with positive diagonal entries ordered in decreasing order, there exists an isometry $g$ of $\spd_n$ with $g\cdot(\Id, D_1)=(\Id, D_2)$ if and only if $D_1=D_2$. }

\textbf{The VVD vector: }To assign this vector in SPD we introduce the vector-valued distance (VVD) function $d_{vv}\colon\spd_n \times \spd_n \to \mathbb{R}^n$.
%which assigns to two points a vector, instead of a scalar. 
For two points $P,Q\in\spd_n$, the VVD is defined as:
%\begin{equation}
%d_{vv}(P,Q) = \rm{log}(\rm{eigenvalues}(P^{-1/2}QP^{-1/2}))  
%\end{equation}
\begin{equation}
d_{vv}(P,Q) = \operatorname{log}(\lambda_1(P^{-1}Q),\ldots,\lambda_n(P^{-1}Q))  
\end{equation}
where $\lambda_1(P^{-1}Q) \geq \ldots \geq \lambda_n(P^{-1}Q)$ are the eigenvalues of $P^{-1}Q$ sorted in descending order.
%is the ordered list of the logarithms of the eigenvalues of the positive definite symmetric matrix $P^{-1/2}QP^{-1/2}$.
This vector is an invariant of the relative position of two points up to isometry. 
This means that in $\spd_n$, only if the VVD between two points $A$ and $B$ is the vector $v \in \mathbb{R}^n$, and the VVD between $P$ and $Q$ is also $v$, then there exists an isometry mapping $A$ to $P$ and $B$ to $Q$. Thus, we can recover completely the relative position of two points in $\spd_n$  from this vector. For example, the Riemannian metric is obtained by using the standard $l_2$ norm on the VVD vector. This is: $d^{R}(P, Q) = ||d_{vv}(P, Q)||_{2}$.
See \cite{kapovich2017vectorValuedDistance} \S2.6 and Appendix~\ref{sec:VecValDist} for a review of VVDs in symmetric spaces. 
%\todo{MS: more space after the caption of Figure 3, looks like cheating to me, or is the problem around Eq.1?}

% \todo[inline]{Cite that the formula for the VVD has also been used/introduced in: \cite{bhatia2003expMetricSPD, mostajeran2020GeometricMM} B: \cite{bhatia2003expMetricSPD} Studies Finsler metrics, but not really VVD, I cited it below. As far as I can see \cite{mostajeran2020GeometricMM} doesn't use VVD, but rather discusses a midpoint, which seems related to convex projective geometry. Since it is far from what we are doing, I wouldn't cite it here (even if the referee asked for it).}

\begin{wrapfigure}{r}{0.65\textwidth}
%\vspace{-5mm}
 \centering
 \includegraphics[width=.64\textwidth,keepaspectratio]{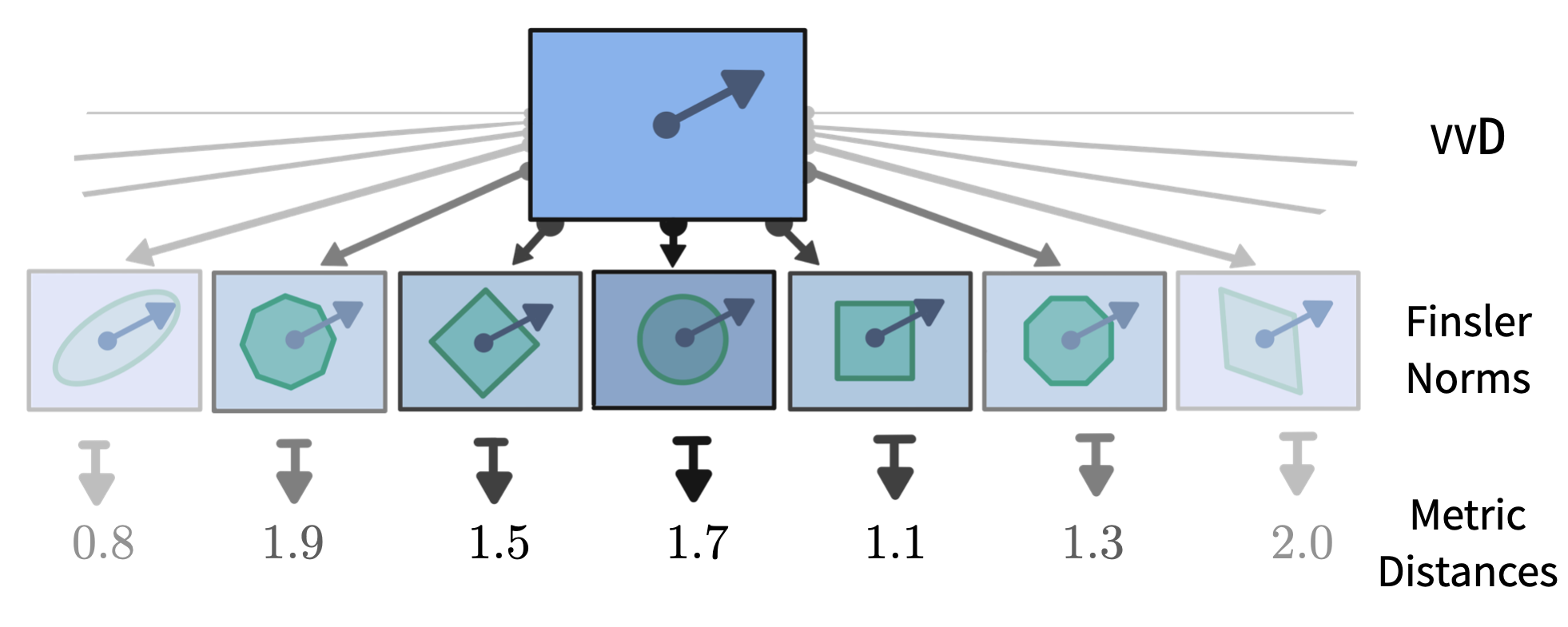}
\caption{The vector-valued distance  allows to reconstruct the Riemannian, or any Finsler distance.}
\label{fig:finslerdist}  
\vspace{-1mm}
\end{wrapfigure}
\textbf{Finsler metrics:} Any norm on $\mathbb{R}^n$ that is invariant under permutation of the entries induces a metric on $\spd_n$. 
Moreover, $\spd_n$ do not only support a Riemannian metric, but also Finsler metrics, a whole family of distances with the same symmetry group (group of isometries). 
These metrics are of special importance since distance minimizing geodesics are not necessarily unique in Finsler geometry. Two different paths can have the same minimal length. This is particularly valuable when embedding graphs in $\spd_n$, since in graphs there are generally several shortest paths.
We obtain the Finsler metrics $\rm{F}_1$ or $\rm{F}_{\infty}$ by taking the respective $\ell_1$ or $\ell_{\infty}$ norms of the VVD in $\mathbb{R}^n$ (see Figure~\ref{fig:finslerdist}). 
See \citet{planche1995finslerMetrics} and Appendix~\ref{sec:FinslerDist} for a review of the theory of Finsler metrics, \cite{bhatia2003expMetricSPD} for the study of some Finsler metrics on $\spd_n$, and \cite{lopez2021symmetric} for applications of Finsler metrics on symmetric spaces in representation learning. 
% motivate finsler metrics
%\todo{ST:Should we reference our old paper('s appendix?) as a spot that talks more about Finsler metrics? Or should I put some of that information into this appendix too?}

%It turns out in experiments, that Finsler metrics often provide a better performance than the Riemannian metric.  

\textbf{Advantages of VVD:} The proposed metric learning framework based on the VVD offers several advantages.
First, a single model can be run with different metrics, according to the chosen norm. The VVD contains the full information of the Riemannian distance and of all invariant Finsler distances, hence we can easily recover the Riemannian metric and extend the approach to many other alternatives (in Appendix~\ref{sec:app-relation-with-other-metrics}, we detail how the VVD generalizes other $\spd_n$ metrics).
Second, the VVD provides much more information than just the distance, and can be used to analyze the learned representation in $\spd_n$, independent of the choice of the metric. 
Out of the VVD between two points, one can immediately read the regularity of the unique geodesics joining these two points. Geodesics in $\spd_n$ have different regularity, which is related with the number of maximal Euclidean subspaces that contain the geodesic. The Riemannian or Finsler distances cannot distinguish the differences between these geodesics of different regularity, but the VVD function can.
Third, the VVD function can be leveraged as a tool to visualize and analyze the learned high-dimensional representations (see \S\ref{sec:analysis}).

\section{Gyrocalculus} 
\label{sec:gyrocalculus}

\vspace{-1mm}

% \todo[inline]{From Reviewer 1: "After searching the literature, it became clear that the same gyro vector calculus presented here has already been introduced and discussed in the pure mathematics literature. This needs to be made more clear in the text and relevant papers need to be cited and given due credit."

% CITE: 
% \cite{abe2015generalizedGyrovectors, hatori2017examplesGyrovector}
% This one get the same formulas than us, right? \cite{kim2016gyrovectorOnSPD}.
% Do they also get the gyromidpoint? With that you have attention on SPD!

% B: I didn't find \cite{abe2015generalizedGyrovectors}, \cite{hatori2017examplesGyrovector} is rather abstract, \cite{kim2016gyrovectorOnSPD} seems the most relevant, and should be cited here - see below.
%}

To build an analog of many Euclidean operators in $\spd_n$, we require also a translation of operations internal to Euclidean geometry, chief among these being the vector space operations of addition and scalar multiplication.
By means of tools introduced in pure mathematics literature\footnote{See \cite{abe2015generalizedGyrovectors, hatori2017examplesGyrovector} for an abstract treatment of gyrocalculus, and \cite{kim2016gyrovectorOnSPD} for the specific examples discussed here.}, we describe a gyro-vector space structure on $\spd_n$, which provides geometrically meaningful extensions of these vector space operations, extending successful applications of this framework in geometric deep learning to hyperbolic space \cite{ganea2018hyperNN,lopez2020fullyhyper, shimizu2021hyperNNplusplus}.
These operations provide a template for translation, where one may
attempt to replace $+,-,\times$ in formulas familiar from Euclidean spaces with the analogous operations $\oplus,\ominus,\otimes$ on $\spd_n$.
While straightforward, such translation requires some care, as gyro-addition is neither commutative nor associative.
See Appendix~\ref{sec:appendix-gyrocalc} for a review of the underlying general theory of gyrogroups and additional guidelines for accurate formula translation. %{\color{blue} The gyrocalculus presented here has already been introduced in the pure mathematics literature. See \cite{abe2015generalizedGyrovectors, hatori2017examplesGyrovector} for an abstract treatment, and \cite{kim2016gyrovectorOnSPD} for the specific example we are discussing here. }

\textbf{Addition and Subtraction:}
%\subsection{Addition and Subtraction}
\label{sec:gyrogroup-structure-gyroaddition}
% The gyrovector space structure of hyperbolic geometry exploited by \cite{ganea2018hyperNN, bachmann2020ccgcn, shimizu2021hyperNNplusplus} arises from physics, where $\oplus$ is the velocity addition operator in special relativity.
% This was given a purely geometric interepretation by 
% \citet{vermeer2005geometricInterpretationGyroaddition} which directly generalizes to a candidate operation for $\oplus$ on $\spd_n$.
Given a fixed choice $I$ of basepoint and two points $P,Q\in \spd_n$, we define the gyroaddition of $P$ and $Q$ to be the point $P\oplus Q\in\spd_n$ which is the image of $Q$ under the isometry which translates $I$ to $P$ along the geodesic connecting them.  This directly generalizes the gyroaddition of hyperbolic space  exploited by \cite{bachmann2020ccgcn, ganea2018hyperNN, shimizu2021hyperNNplusplus}, via the geometric interpretation of \citet{vermeer2005geometricInterpretationGyroaddition} (see Figure~\ref{fig:gyro-operations}).

Fixing $P\in\spd_n$, we may compute the value of $P\oplus Q$ for arbitrary $Q$ as the result of applying the $\spd_n$-translation moving the basepoint to $P$, evaluated on $Q$.  We see also the additive inverse of a point with respect to this operation must then be given by its geodesic reflection in $I$.
\begin{equation}
\label{eq:gyro-addition}
      P\oplus Q =\sqrt{P}Q\sqrt{P}\hspace{1cm}\ominus P = P^{-1}
\end{equation}
As this operation encodes a symmetry of $\spd_n$, it is possible to recast certain geometric statements purely in the gyrovector formalism.  In particular, the vector-valued distance $d_{vv}(P,Q)$ may be computed as the logarithm of the eigenvalues of $\ominus P\oplus Q$ (see Appendix~\ref{sec:RiemDist}).

\begin{figure*}[!t]
\vspace{-2mm}
\centering
\subfloat{\includegraphics[width=.32\textwidth,keepaspectratio]{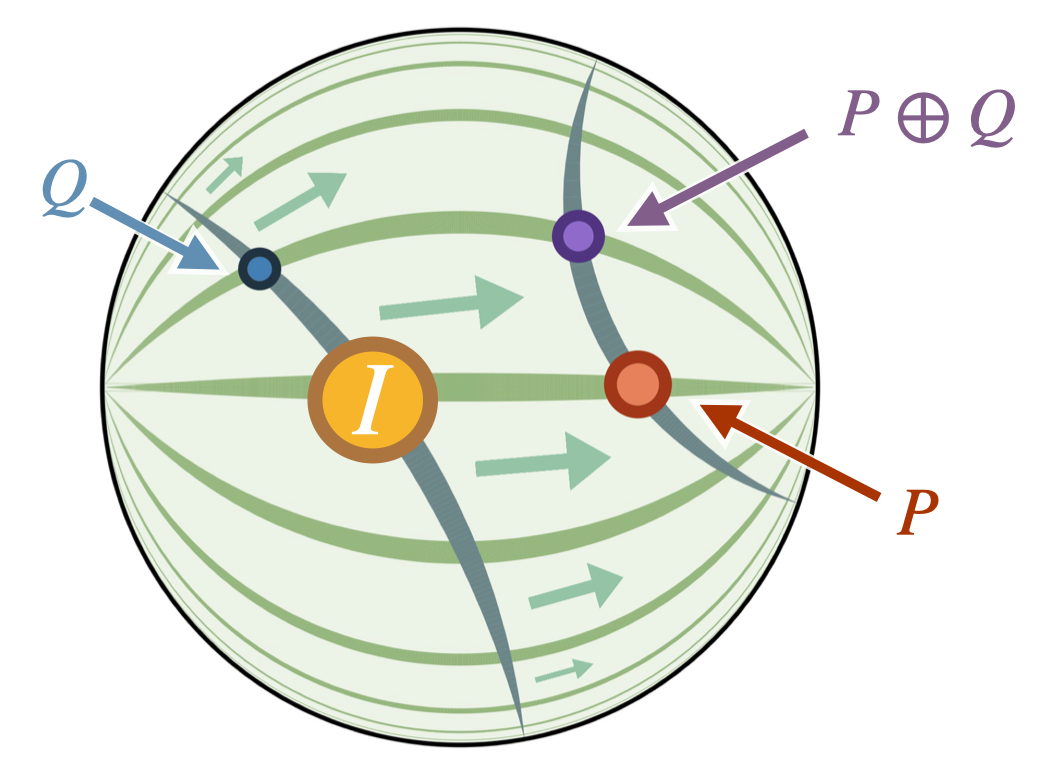}}\hfill
\subfloat{\includegraphics[width=.32\textwidth,keepaspectratio]{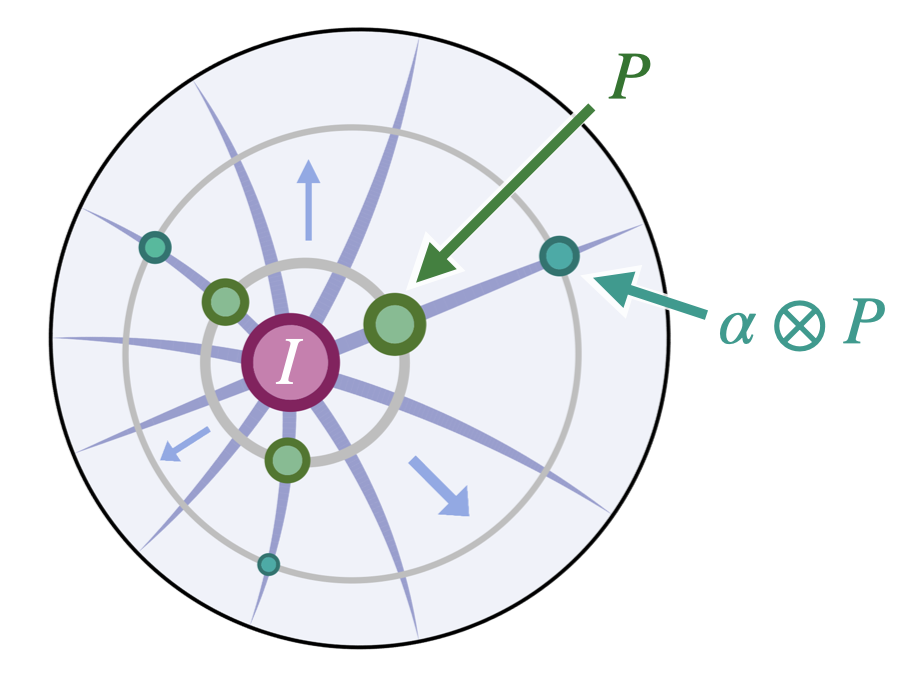}}\hfill
\subfloat{\includegraphics[width=.32\textwidth,keepaspectratio]{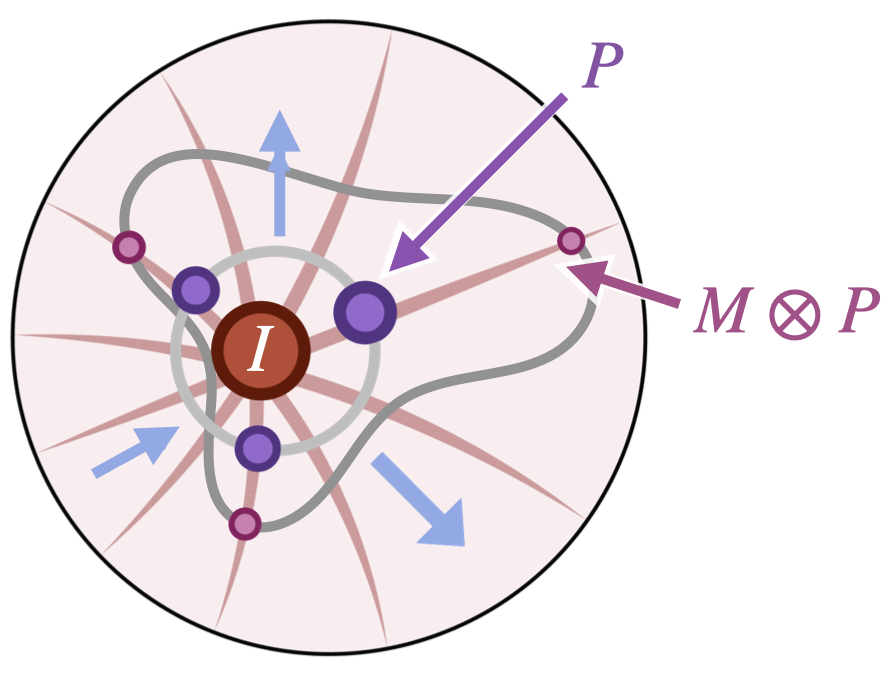}}\hfill
\caption{Gyro-addition (left), gyro-scalar multiplication (center) and matrix scaling (right).}
\label{fig:gyro-operations}
\vspace{-4mm}
\end{figure*}

\textbf{Scalar Multiplication and Matrix Scaling:}
\label{sec:gyro-matrix-scaling}
For a fixed basepoint $I$, we define the scalar multiplication of a point $P\in\spd_n$ by a scalar $\alpha\in \RR_+$ to be the point which lies at distance $\alpha d(I,P)$ from $I$ in the direction of $P$, where $d(\cdot,\cdot)$ is the metric distance on $\spd_n$.
Geometrically, this is a transfer of the vector-space scalar multiplication on the tangent space to $\spd_n$:
% That is, we think of the operation $\alpha\otimes -$ as geometrically analogous to standard scalar multiplication on $\RR^n$: upon multiplication by $\alpha$, each point of $\spd_n$ is moved $\alpha$ times farther away from the basepoint $I$.
%We may again use the geometry of $\spd_n$ to give a concrete formula for this geometric definition:
\begin{equation}
\label{eq:scalar-multiplication}
    \alpha\otimes P=P^\alpha=\exp(\alpha\log(P)),
\end{equation}
where $\exp,\log$ are the matrix exponential and logarithm. %, readily computable via orthogonal diagonalization (see \S\ref{s.orthdiag}).
We further generalize the notion of scalar multiplication to allow for different relative expansion rates in different directions. For a fixed basepoint $I$ and a point $P \in \spd_n$, we can replace the scalar $\alpha$ from Eq.~\ref{eq:scalar-multiplication} with an arbitrary real symmetric matrix $A \in S_n$. We define this \emph{matrix scaling} by:
\begin{equation}
\label{eq:gyro-matrix-scaling}
    A \otimes P = \exp(A \odot \log(P))
\end{equation}
%where $A$ has $\nicefrac{n(n+1)}{2}$ \todo{F to self: Remove/fix this since we are talking about weights and this is not in teh Implementation part}{learnable weights} given its symmetry, and 
where $A \odot X$ denotes the Hadamard product. We denote the matrix scaling with $\otimes$, extending the previous usage: for any $\alpha\in \RR$, we have $[\alpha]\otimes P=\alpha\otimes P$ where $[\alpha]$ is the matrix with every entry $\alpha$.

\section{Implementation}
%In this section we apply the symmetric structure of the SPD manifold and the gyro-vector operations to implement different linear mappings in the space. 
In this section we detail how we learn representations in $\spd_n$, and implement different linear mappings so that they conform to the premises of each operator, yielding SPD neural layers.

\textbf{Embeddings in $\spd_n$ and $S_n$:}
%\subsection{Points in $\spd_n$ and $S_n$}
\label{sec:points-in-spd}
We are interested in learning embeddings in $\spd_n$. To do so we exploit the connection between $\spd_n$ and its tangent space $S_n$ through the exponential and logarithmic maps.
To learn a point $P \in \spd_n$, we first model it as a symmetric matrix $U \in S_n$. We impose symmetry on $U$ by learning a triangular matrix $X \in \RR^{n \times n}$ with $\nicefrac{n(n+1)}{2}$ parameters, such that $U = X + X^T$.
To obtain the matrix $P \in \spd_n$, we employ the exponential map: $P = \exp(U)$. 
Modeling points on the tangent space offers advantages for optimization, explained in \S\ref{sec:optimization}.
For the matrix scaling $A \otimes P$, we impose symmetry on the factor matrix $A \in S_n$ in the same way that we learn the symmetric matrix $U$.

\textbf{Isometries: Rotations and Reflections: }
%\subsection{Isometries: Rotations and Reflections}
%We recall from \S\ref{sec:symmetries-of-spd} that the isometry $P \mapsto MPM^T$ defined by an orthogonal matrix $M\in O(n)$ fixes the basepoint, and may be thought of as a generalized rotation or reflection.
%To use such transformations to encode relations, we describe here a means of specifying elements of $O(n).$\todo{B: Skip this paragraph, and move the matrix to the appendix?}
Rotations in $n$ dimensions are described as collections of pairwise orthogonal $2$-dimensional rotations in planes (with a leftover 1-dimensional "axis of rotation" in odd dimensions).  
We utilize this observation to efficiently build elements of $O(n)$ out of two-dimensional rotations in coordinate planes.
More precisely, for any $\theta\in[0,2\pi)$ and choice of sign $\{+,-\}$ we let $R^\pm(\theta)$ denote the 2-dimensional rotation ($+$) or reflection ($-$) as 
$R^\pm(\theta)=\left(\begin{smallmatrix}
\cos \theta&\mp\sin\theta\\
\sin\theta&\pm\cos\theta
\end{smallmatrix}\right)
$.  
\begin{wrapfigure}{r}{0.4\textwidth}
\vspace{-3mm}
\centering
\small
$$R_{24}^+(\theta)=\begin{pmatrix}
1&0&0&0&0\\
0&\cos\theta&0&-\sin\theta&0\\
0&0&1&0&0\\
0&\sin\theta&0&\cos\theta&0\\
0&0&0&0&1
\end{pmatrix}$$
\label{fig:matrix-rel}
\vspace{-3mm}
\end{wrapfigure}
Then for any pair $i<j$ in $1\ldots n$, we denote by $R_{ij}^\pm(\theta)$ the transformation which applies $R^\pm(\theta)$ to the $x_ix_j$-plane of $\RR^n$, and leaves all other coordinates fixed. 
For example, in $O(5)$ the element $R_{24}^+(\theta)$ (see on the right) denotes the transformation where we replace the entries $(ii, ij, ji, jj)$ of $I_n$ with the corresponding values of $R^+(\theta)$. More general, rotations and reflections are built by taking products of these basic transformations. Given a $\nicefrac{n(n-1)}{2}$-dimensional vector of angles $\vec{\theta}=(\theta_{12},\ldots,\theta_{ij},\ldots,)$ and a choice of sign, we define the rotation and reflection corresponding to $\vec{\theta}$ by:
\begin{equation}
\label{eq:rot-ref-matrix-construction}
    \operatorname{Rot}(\vec{\theta})=\prod_{i<j}R^+_{ij}(\theta_{ij})\hspace{1cm} \operatorname{Refl}(\vec{\theta})=\prod_{i<j}R^-_{ij}(\theta_{ij})
\end{equation}
where $\operatorname{Rot}(\vec{\theta}), \operatorname{Ref}(\vec{\theta}) \in \RR^{n \times n}$ are the isometry matrices, and the vector of angles $\vec{\theta}$ can be regarded as a learnable parameter of the model. Finally, we  denote the application of the transformation $M$ to the point $P \in \spd_n$ by:
\begin{equation}
    \label{eq:rotref}
    M \circledcirc P = MPM^T
\end{equation}
%We regard the $\binom{n}{2}$-dimensional vector of angles $\vec{\theta}$ as learnable weights of the model.

\textbf{Optimization:}
%\subsection{Optimization} 
\label{sec:optimization}
For the proposed rotations and reflections, the learnable weights are vectors of angles $\vec{\theta} \in \RR^{\frac{n(n-1)}{2}}$, which do not pose an optimization challenge.
On the other hand, embeddings in $\spd$ have to be optimized respecting the geometry of the manifold, but as already explained, we model them on the space of symmetric matrices $S_n$, and then we apply the exponential map. In this manner, we are able to perform tangent space optimization \cite{chami2019hgcnn} using standard Euclidean techniques, and circumvent the need for Riemannian optimization \cite{bonnabel2011rsgd, becigneul2019riemannianMethods}, which we found to be less numerically stable. Due to the geometry of $\spd_n$ (see Appendix~\ref{sec:exponential-maps}), this is an exact procedure, which does not incur losses in representational power.

%In this manner, we can use standard Euclidean optimization techniques, and respect the geometry of the manifold. 
\textbf{Complexity:} The most frequently utilized operation when learning embeddings is the distance calculation, thus we analyze its complexity. In Appendix~\ref{app:dist-alg-complexity} we detail the complexity of different operations.
Calculating the distance between two points in $\spd_n$ implies computing multiplications, inversions and diagonalizations of $n \times n$ matrices. We find that the cost of the distance computation with respect to the matrix dimensions is $\bigO(n^3)$. Although a matrix of rank $n$ implies $n(n+1)/2$ dimensions thus a large $n$ value is usually not required, the cost of many operations is polynomial instead of linear.

\textbf{Towards neural network architectures:}
We employ the proposed mappings along with the gyro-vector operations %introduced in the previous section in \S\ref{sec:gyrogroup-structure-gyroaddition}, 
as building blocks for SPD neural layers. This is showcased in the experiments presented below.
Scalings, rotations and reflections can be seen as feature transformations. Moreover, gyro-addition allows us to define the equivalent of bias addition. Finally, although we do not employ non-linearities, our approach can be seamlessly integrated with the ReEig layer (adaptation of a ReLU layer for SPD) proposed in \cite{huang2017riemannianNetForSPDMatrix}.

%Gyrocalculus can be used to construct neural network architectures on SPD manifolds.  The different linear mappings, along with the gyro-vector operations, can be seen as feature transformations and employed as building blocks for SPD neural models. 

\section{Experiments}
\label{sec:experiments}

In this section we employ the transformations developed on SPD to build neural models for knowledge graph completion, item recommendation and question answering. Task-specific models in different geometries have been developed in the three cases, hence we consider them adequate benchmarks for representation learning.\footnote{Code available at \url{https://github.com/fedelopez77/gyrospd}}

\subsection{Knowledge Graph Completion}
\label{sec:kg-completion}
Knowledge graphs (KG) %are multi-relational graphs where nodes represent entities and typed-edges represent relationships among entities. They are popular data structures for 
represent heterogeneous knowledge in the shape of \textit{(head, relation, tail)} triples, where \textit{head} and \textit{tail} are entities and \textit{relation} represents a relationships among entities. 
KG exhibit an intricate and varying structure %as a result of the logical properties of the relationships they encode \cite{miller1995wordnet, Suchanek2007yago, lehmann2015dbpedia}. 
where entities can be connected by symmetric, anti-symmetric, or hierarchical relations. To capture these non-trivial patterns more expressive modelling techniques become necessary \cite{chami2020lowdimkge}, thus we choose this application to showcase the capabilities of our transformations on SPD manifolds.
Given an incomplete KG, the task is to predict which unknown links are valid.

\textbf{Problem formulation:}
Let $\mathcal{G} = (\EntitySub,\RelSub,\TripSub)$ be a knowledge graph where $\EntitySub$ is the set of entities, $\RelSub$ is the set of relations and $\TripSub \subset \EntitySub \times \RelSub \times \EntitySub$ is the set of triples stored in the graph. 
The usual approach is to learn a scoring function $\phi : \EntitySub \times \RelSub \times \EntitySub \rightarrow \R$ that measures the likelihood of a triple to be true, with the goal of scoring all missing triples correctly.
To do so, we propose to learn representations of entities as embeddings in $\spd_n$, and relation-specific transformation in the manifold, such that the KG structure is preserved.
%Most of the KG embedding methods learn vectors $\operatorname{\textbf{h}}, \operatorname{\textbf{t}} \in \R^{n_{\EntitySub}}$ for $h, t \in \EntitySub$, and $\textbf{r} \in \R^{n_{\RelSub}}$ for $r \in \RelSub$.

%\footnote{We use small italics characters (e.g., \textit{h}, \textit{r}) to represent entities and relations, and corresponding bold characters to represent their vector embeddings (e.g., \textbf{h}, \textbf{r}).} 

%We experiment with two variations of a scoring functions that has shown success in the task. It is proposed for the hyperbolic model MURP \cite{balazevic2019murp}, and combines multiplicative and additive components, which are fundamental to model different properties of KG relations \cite{allenBalazevic2021interpretingKG}.
\textbf{Scaling model:} We follow the base hyperbolic model MuRP \cite{balazevic2019murp} and adapt it into $\spd_n$ by means of the \textit{matrix scaling}. Its scoring function has shown success in the task given that it combines multiplicative and additive components, which are fundamental to model different properties of KG relations \cite{allenBalazevic2021interpretingKG}. We translate it into $\spd_n$ as:
\begin{equation}
\label{eq:kg-scaling-scoring_fn}
    \phi(h, r, t) = -d((\textbf{M}_r \otimes \textbf{H}) \oplus \textbf{R}, \textbf{T})^{2} + b_h + b_t
\end{equation}
where $\textbf{H}, \textbf{T} \in \spd_n$ are embeddings and $b_h, b_t \in \RR$ are scalar biases for the head and tail entities respectively. $\textbf{R} \in \spd_n$ and $\textbf{M}_r$ are matrices that depend on the relation. For $d(\cdot,\cdot)$, we experiment with the Riemannian and the Finsler One metric distances.
%For $\textbf{M}_r$ we consider two alternatives. The first one is a scaling transformation on the tangent space where $\textbf{M}_r \in \Sym_n$ and $\textbf{M}_r \otimes \textbf{H} = \exp(\textbf{M}_r \odot \log(\textbf{H}))$ with $\odot$ denoting the Hadamard product.

\textbf{Isometric model:} A possible alternative is to embed the relation-specific transformations as elements of the $O(n)$ group (\textit{i.e.}, rotations and reflections). This technique has proven effective in different metric spaces \cite{chami2020lowdimkge, yang2020NagEnonAbelianGroupsKG}. In this case, $\textbf{M}_r$ is a rotation or reflection matrix as in Eq.~\ref{eq:rot-ref-matrix-construction}, and the scoring function is defined as:
\begin{equation}
\label{eq:kg-isometry-scoring_fn}
    \phi(h, r, t) = -d((\textbf{M}_r \circledcirc \textbf{H}) \oplus \textbf{R}, \textbf{T})^{2} + b_h + b_t
\end{equation}

%\subsubsection{Setup}
\textbf{Datasets:} We employ two standard benchmarks, namely WN18RR \cite{bordes2013transe, dettmers2018conve} and FB15k-237 \cite{bordes2013transe, toutanova2015kgObserveFeats}. WN18RR  is a subset of WordNet \cite{miller1995wordnet} containing $11$ lexical relationships between $40,943$ word senses. %, and has a natural hierarchical structure, e.g., (car, hypernym of, sedan). 
FB15k-237 is a subset of Freebase \cite{bollacker2008freebase}, %a collaborative knowledge base of general world knowledge. It has 
with $14,541$ entities and $237$ relationships.

\begin{wrapfigure}{r}{0.45\textwidth}
\vspace{-2mm}
%\centering
\small
\begin{equation}
\label{eq:cross-entropy-loss}
\mathcal{L} = \sum_{(h, r, t) \in \mathcal{T}} \log(1 + \operatorname{exp}(Y_{t} \phi(h, r, t)))
\end{equation}
\vspace{-6mm}
\end{wrapfigure}
\textbf{Training:} 
We follow the standard data augmentation protocol by adding inverse relations to the datasets \cite{lacroix2018tensordecomp}. 
We optimize the cross-entropy loss with uniform negative sampling defined in Equation~\ref{eq:cross-entropy-loss},
where $\mathcal{T}$ is the set of training triples, and $Y_{t} = -1$ if $t$ is a factual triple or $Y_{t} = 1$ if $t$ is a negative sample.
We employ the AdamW optimizer \cite{loshchilov2018adamW}.
We conduct a grid search with matrices of dimension $n \times n$ where $n \in \{14, 20, 24\}$ (this is the equivalent of $\{105, 210, 300\}$ degrees of freedom respectively) to select optimal dimensions, learning rate and weight decay, using the validation set. More details and set of hyperparameters in Appendix~\ref{app:expdetails-kgcompletion}.

\textbf{Evaluation metrics:} At test time, we rank the correct tail or head entity against all possible entities using the scoring function, and use inverse relations for head prediction \cite{lacroix2018tensordecomp}. Following previous work, we compute two ranking-based metrics: mean reciprocal rank (MRR), which measures the mean of inverse ranks assigned to correct entities, and hits at K (H@K, $K \in \{1, 3, 10\}$), which measures the proportion of correct triples among the top K predicted triples. We follow the standard evaluation protocol of filtering out all true triples in the KG during evaluation \cite{bordes2013transe}.
%, since predicting a low rank for these triples should not be penalized \cite{bordes2013transe}.

\textbf{Baselines:} We compare our models with their respective equivalents in different metric spaces, which are also state-of-the-art models for the task. 
%state-of-the-art baselines that perform the equivalent transformations in different metric spaces: 
For the scaling model, these are \textsc{MuRE} and \textsc{MuRP} \cite{balazevic2019murp}, which perform the scaling operation in Euclidean and hyperbolic space respectively. For the isometric models, we compare to \textsc{RotC} \cite{sun2018rotate}, \textsc{RotE} and \textsc{RotH}, \cite{chami2020lowdimkge} (rotations in Complex, Euclidean and hyperbolic space respectively), and \textsc{RefE} and \textsc{RefH} \cite{chami2020lowdimkge} (reflections in Euclidean and hyperbolic space). 
%Moreover, we additionally compare with other state-of-the-art models, including ComplEx \cite{trouillon2016complex}, Tucker \cite{balazevic2019tucker}, and Quaternion \cite{zhang2019quaternionKG}. 
Baseline results are taken from the original papers.
We do not compare to previous work on SPD given that they lack the definition of an arithmetic operation in the space, thus a vis-a-vis comparison is not possible. %\todo{F: removable\\B: I would remove it}{Moreover}, the definition of the transformation layers employed in \cite{dong2017spdToFaceRecog, gao2019robustRepreSPD, huang2017riemannianNetForSPDMatrix} requires optimizing over compact Stiefel manifolds, plus the derivation of a particular Riemannian matrix backpropagation rule, which is a highly non-trivial and impractical approach, whereas our implementation employs off-the-shelf optimizers.

\textbf{Results:}
We report the performance for all analyzed models, segregated by operation, in Table~\ref{tab:kg-results}.
%For WN18RR: Brief comment on which transformation is better. Comment on Finsler metric seems to be slightly better.
On both dataset, the scaling model $\operatorname{SPD}_{\operatorname{Sca}}$ outperforms its direct competitors MuRE and MuRP, and this is specially notable in HR@10 for WN18RR: $59.0$ for \SPDtraFone vs $55.4$ and $56.6$ respectively. 
SPD reflections are very effective on WN18RR as well. They outperform their Euclidean and hyperbolic counterparts RefE and RefH, in particular when equipped with the Finsler metric.
Rotations on the SPD manifold, on the other hand, seem to be less effective. However, Euclidean and hyperbolic rotations require $500$ dimensions whereas the $\operatorname{SPD}_{\operatorname{Rot}}$ models are trained on matrices of rank $14$ (equivalent to $105$ dims).
Moreover, the underperformance observed in some of the analyzed cases for rotations and reflections does not repeat in the following experiments (\S\ref{sec:kg-rs-exps} \& \S\ref{sec:qa-exps}). Hence, we consider this is due to overfitting in some particular setups. Although we tried different regularization methods, we regard a sub-optimal configuration rather than a geometric reason to be the cause for the underperformance.

Regarding the choice of a distance metric, the Finsler One metric is better suited with respect to HR@3 and HR@10 when using scalings and reflections on WN18RR.
For the FB15k-237 dataset, SPD models operating with the Riemannian metric outperform their Finsler counterparts. This suggests that the Riemannian metric is capable of disentangling the large number of relationships in this dataset to a better extent.

In these experiments we have evaluated models applying equivalent operations and scoring functions in different geometries, thus they can be thought as a vis-a-vis comparison of the metric spaces. We observe that SPD models tie or outperform baselines in most instances. This showcases the improved representation capacity of the SPD manifold when compared to Euclidean and hyperbolic spaces. Moreover, it demonstrates the effectiveness of the proposed metrics and operations in this manifold.

\begin{table}[!t]
\vspace{-4mm}
\caption{Results for Knowledge graph completion.}
\label{tab:kg-results}
\small
\centering
\adjustbox{max width=0.85\linewidth}{
\begin{tabular}{llrrrrrrrr}
\toprule
 &  & \multicolumn{4}{c}{WN18RR} & \multicolumn{4}{c}{FB15k-237} \\
 \cmidrule(lr){3-6}\cmidrule(lr){7-10}
Operation & Model & \multicolumn{1}{l}{MRR} & \multicolumn{1}{l}{HR@1} & \multicolumn{1}{l}{HR@3} & \multicolumn{1}{l}{HR@10} & \multicolumn{1}{l}{MRR} & \multicolumn{1}{l}{HR@1} & \multicolumn{1}{l}{HR@3} & \multicolumn{1}{l}{HR@10} \\
\midrule
\multirow{4}{*}{Scaling} & \textsc{MuRE} & 47.5 & 43.6 & 48.7 & 55.4 & 33.6 & 24.5 & 37.0 & 52.1 \\
 & \textsc{MuRP} & 48.1 & \textbf{44.0} & 49.5 & 56.6 & 33.5 & 24.3 & 36.7 & 51.8 \\
 & \SPDtraRiem & 48.1 & 43.1 & 50.1 & 57.6 & \textbf{34.5} & \textbf{25.1} & \textbf{38.0} & \textbf{53.5} \\
 & \SPDtraFone & \textbf{48.4} & 42.6 & \textbf{51.0} & \textbf{59.0} & 32.9 & 23.6 & 36.3 & 51.5 \\
 \midrule
\multirow{5}{*}{Rotations} & \textsc{RotC} & 47.6 & 42.8 & 49.2 & 57.1 & 33.8 & 24.1 & 37.5 & 53.3 \\
 & \textsc{RotE} & 49.4 & 44.6 & 51.2 & 58.5 & \textbf{34.6} & \textbf{25.1} & \textbf{38.1} & \textbf{53.8} \\
 & \textsc{RotH} & \textbf{49.6} & \textbf{44.9} & \textbf{51.4} & \textbf{58.6} & 34.4 & 24.6 & 38.0 & 53.5 \\
 & \SPDrotRiem & 46.2 & 39.7 & 49.6 & 57.8 & 32.9 & 23.6 & 36.3 & 51.6 \\
 & \SPDrotFone & 40.9 & 30.5 & 48.2 & 57.3 & 32.1 & 22.9 & 35.4 & 50.5 \\
 \midrule
\multirow{4}{*}{Reflections} & \textsc{RefE} & 47.3 & 43.0 & 48.5 & 56.1 & \textbf{35.1} & \textbf{25.6} & \textbf{39.0} & \textbf{54.1} \\
 & \textsc{RefH} & 46.1 & 40.4 & 48.5 & 56.8 & 34.6 & 25.2 & 38.3 & 53.6 \\
 & \SPDrefRiem & 48.3 & 44.0 & 49.7 & 56.7 & 32.5 & 23.4 & 35.6 & 51.0 \\
 & \SPDrefFone & \textbf{48.7} & \textbf{44.3} & \textbf{50.1} & \textbf{57.4} & 31.6 & 22.5 & 34.6 & 50.0 \\
 \bottomrule
\end{tabular}
}
\vspace{-5mm}
\end{table}

\subsection{Knowledge Graph Recommender Systems}
\label{sec:kg-rs-exps}
Recommender systems (RS) model user preferences to provide personalized recommendations \cite{zhang2019dlrssurvey}.
KG embedding methods have been widely adopted into the recommendation problem as an effective tool to model side information and enhance the performance \cite{zhang2016collkbe,guo2020kgsurvey}. 
For instance, one reason for recommending a movie to a particular user is that the user has already watched many movies from the same genre or director \cite{ma2019explainableRuleskgrecosys}.
Given multiple relations between users, items, and heterogeneous entities, the goal is to predict the user's next item purchase or preference. 
%can be mapped into the KG and incorporated to alleviate data sparsity and enhance the recommendation performance.

\textbf{Model:} We model the recommendation problem as a link prediction task over users and items \cite{li2014rslinkpred}. In addition, we aim to incorporate side information between users, items and other entities. Hence we apply our KG embedding method from \S\ref{sec:kg-completion} as is, to embed this multi-relational graph. We evaluate the capabilities of the approach by only measuring the performance over user-item interactions.

\textbf{Datasets:} To investigate the recommendation problem endowed with added relationships, we employ the Amazon dataset \cite{mcauley2013hft, nimcauley2019extamazondata} (branches "Software", "Luxury \& Beauty" and "Prime Pantry"), with users' purchases of products, and the MindReader dataset \cite{brams2020mindreader} of movie recommendations. 
Both datasets provide additional relationships between users, items and entities such as product brands, or movie directors and actors.
% User-item interactions are users' purchases of products. For each purchase it provides the review written by the user about the item, and item metadata in the form of textual descriptions. %, product brand and categorical labels. %Moreover, the dataset also contains relationships between pairs of items, such as items that are usually viewed or bought together.
To generate evaluation splits, the penultimate and last item the user has interacted with are withheld as dev and test sets respectively.

\textbf{Training:} In this setup we also augment the data by adding inverse relations and optimize the loss from Equation~\ref{eq:cross-entropy-loss}.
We set the size of the matrices to $10 \times 10$ dimensions (equivalent to $55$ free parameters). %We conduct a grid search to select optimal batch size and learning rate, using the validation set. We report the average of $3$ runs. 
More details about relationships and set of hyperparameters in Appendix~\ref{app:expdetails-kgrecosys}.

\textbf{Evaluation and metrics:} 
%Since it is very costly to rank all the available items, we randomly select $100$ samples which the user has not interacted with, and rank the ground truth amongst these samples \cite{he2017neuralCF}.
We follow the standard procedure of evaluating against $100$ randomly selected samples the user has not interacted with \cite{he2017neuralCF, lopez2021relco}.
To evaluate the recommendation performance we focus on the \textit{buys / likes} relation. For each user $u$ we rank the items $i_{j}$ according to the scoring function $\phi(u, buys, i_{j})$.
We adopt MRR and H@10, as ranking metrics for recommendations.

\textbf{Baselines:} We compare to TransE \cite{bordes2013transe}, RotC \cite{sun2018rotate}, MuRE and MuRP \cite{balazevic2019murp} trained with $55$ dimensions.

\begin{table}[!t]
\vspace{-4mm}
\caption{Results for Knowledge graph-based recommender systems.}
\label{tab:recosys-results}
\small
\centering
\adjustbox{max width=0.95\linewidth}{
\begin{tabular}{lcccccccc}
\toprule
 & \multicolumn{2}{c}{\textsc{Software}} & \multicolumn{2}{c}{\textsc{Luxury}} & \multicolumn{2}{c}{\textsc{Pantry}} & \multicolumn{2}{c}{\textsc{MindReader}} \\
Model & MRR & H@10 & MRR & H@10 & MRR & H@10 & MRR & H@10 \\
\cmidrule(lr){2-3}\cmidrule(lr){4-5}\cmidrule(lr){6-7}\cmidrule(lr){8-9}
\textsc{TransE} & 28.5$\pm$0.1 & 47.2$\pm$0.5 & 35.6$\pm$0.1 & 52.3$\pm$0.1 & 16.6$\pm$0.0 & 35.3$\pm$0.1 & 19.1$\pm$0.4 & 37.6$\pm$0.1 \\
\textsc{RotC} & 28.5$\pm$0.3 & 45.4$\pm$1.4 & 33.0$\pm$0.1 & 49.8$\pm$0.2 & 14.5$\pm$0.0 & 31.3$\pm$0.2 & 25.3$\pm$0.3 & 50.3$\pm$0.6 \\
\textsc{MuRE} & 29.4$\pm$0.4 & 47.1$\pm$0.4 & 35.6$\pm$0.7 & 54.0$\pm$0.3 & 19.4$\pm$0.1 & 39.5$\pm$0.2 & 25.2$\pm$0.3 & 49.9$\pm$0.6 \\
\textsc{MuRP} & 29.6$\pm$0.3 & 47.9$\pm$0.3 & \textbf{37.5$\pm$0.1} & \textbf{55.2$\pm$0.3} & 19.4$\pm$0.1 & 39.8$\pm$0.2 & 25.3$\pm$0.3 & 49.3$\pm$0.2 \\
\SPDtraRiem & 29.4$\pm$0.4 & 48.1$\pm$0.8 & \textbf{37.5$\pm$0.2} & 55.1$\pm$0.2 & 19.5$\pm$0.0 & 39.6$\pm$0.3 & 25.4$\pm$0.1 & 49.8$\pm$0.3 \\
\SPDtraFone & 28.8$\pm$0.1 & 46.9$\pm$0.5 & 37.3$\pm$0.3 & 54.1$\pm$0.9 & 19.0$\pm$0.1 & 38.8$\pm$0.2 & \textbf{25.7$\pm$0.5} & 49.5$\pm$0.1 \\
\SPDrotRiem & \textbf{30.3$\pm$0.2} & 48.6$\pm$0.9 & 37.2$\pm$0.1 & 54.8$\pm$0.4 & \textbf{20.0$\pm$0.1} & \textbf{40.3$\pm$0.1} & 25.3$\pm$0.0 & \textbf{50.5$\pm$0.3} \\
\SPDrotFone & 30.1$\pm$0.1 & \textbf{49.1$\pm$0.3} & 36.9$\pm$0.1 & 54.5$\pm$0.6 & 19.2$\pm$0.0 & 39.3$\pm$0.1 & \textbf{25.7$\pm$0.0} & 49.5$\pm$0.2 \\
\SPDrefRiem & 29.6$\pm$0.2 & 48.0$\pm$0.5 & 37.3$\pm$0.2 & 55.0$\pm$0.2 & 19.3$\pm$0.0 & 39.7$\pm$0.3 & 25.3$\pm$0.0 & 49.1$\pm$0.1 \\
\SPDrefFone & 29.3$\pm$0.1 & 47.5$\pm$0.6 & 36.8$\pm$0.0 & 54.8$\pm$0.1 & 18.6$\pm$0.2 & 38.3$\pm$0.3 & 24.8$\pm$0.2 & 47.9$\pm$1.8 \\
\bottomrule
\end{tabular}
}
\vspace{-4mm}
\end{table}

\textbf{Results:} In Table~\ref{tab:recosys-results} we observe that the SPD models tie or outperform the baselines in both MRR and HR@10 across all analyzed datasets.
Rotations in both, Riemannian and Finsler metrics, are more effective in this task, achieving the best performance in 3 out of 4 cases, followed by the scaling models. 
Overall, this shows the capabilities of the systems to effectively represent user-item interactions enriched with relations between items and their attributes, thus learning to better model users’ preferences. Furthermore, it displays the versatility of the approach to diverse data domains.

\subsection{Question Answering}
\label{sec:qa-exps}
We evaluate our approach on the task of Question Answering (QA). In this manner we also showcase the capabilities of our methods to train word embeddings.

\begin{wrapfigure}{r}{0.4\textwidth}
\vspace{-3mm}
%\centering
\small
\begin{equation}
\begin{split}
    \label{eq:qa-simfunction}
    \operatorname{sim}(q, a) = -w_{f} d(\textbf{Q}, \textbf{A}) + w_{b}, \\
    \text{where } \textbf{Q} = T(\bigoplus\limits_{i=1}^n t_{i}^q) \oplus B
\end{split}
\end{equation}
\vspace{-5mm}
\end{wrapfigure}
\textbf{Model:} {We} adapt the model  from HyperQA \cite{tay2018hyperQA} to $\spd$. We model word embeddings $t_i \in \spd_n$, and represent question/answers as a summation of the embeddings of their corresponding tokens. We apply a feature transformation $T(\cdot)$ followed by a bias addition, as an equivalent of a neural linear layer. $T(\cdot)$ can be a scaling, rotation or reflection. Finally we compute a distance-based similarity function between the resulting question/answer representations as defined in Equation~\ref{eq:qa-simfunction}, where $w_f, w_b \in \RR$, $B \in \spd_n$ and the transformation $T$ are parameters of the model.

\textbf{Datasets:} We analyze two popular benchmarks for QA: TrecQA \cite{wang2007jeopardyTrecqaDataset} (clean version) and WikiQA \cite{yang2015wikiqaDataset}, filtering out questions with multiple answers from the dev and test sets.

\begin{wraptable}{r}{0.5\linewidth}
	%\begin{table}[H]
	\vspace{-3.5mm}
%\begin{table}[!t]
\caption{Results for Question Answering.}
\label{tab:qa-results}
\small
\centering
\adjustbox{max width=\linewidth}{
\begin{tabular}{lcccc}
\toprule
 & \multicolumn{2}{c}{\textsc{TrecQA}} & \multicolumn{2}{c}{\textsc{WikiQA}} \\
Model & MRR & H@1 & MRR & H@1 \\
\cmidrule(lr){2-3}\cmidrule(lr){4-5}
$\operatorname{Euclidean}$ & 55.9$\pm$2.0 & 41.0$\pm$2.0 & 43.4$\pm$0.3 & 22.4$\pm$1.1 \\
$\operatorname{Hyperbolic}$ & 58.0$\pm$1.3 & 39.3$\pm$2.0 & 44.0$\pm$0.4 & 22.8$\pm$0.6 \\
\SPDtraRiem & 55.4$\pm$0.1 & 37.1$\pm$0.1 & \textbf{45.5$\pm$0.5} & 24.4$\pm$1.1 \\
\SPDtraFone & 57.1$\pm$0.7 & 38.6$\pm$0.2 & 44.8$\pm$0.5 & 24.0$\pm$0.6 \\
\SPDrotRiem & 58.7$\pm$1.5 & 41.4$\pm$2.9 & 44.6$\pm$0.6 & 23.6$\pm$0.6 \\
\SPDrotFone & 58.1$\pm$0.5 & \textbf{43.6$\pm$1.0} & 43.7$\pm$0.4 & 23.8$\pm$0.8 \\
\SPDrefRiem & 57.3$\pm$0.3 & 40.7$\pm$1.1 & 43.9$\pm$0.7 & 23.4$\pm$2.0 \\
\SPDrefFone & \textbf{59.6$\pm$0.5} & 42.1$\pm$1.0 & 44.7$\pm$1.2 & \textbf{25.0$\pm$2.5} \\
\bottomrule
\end{tabular}
}
\vspace{-4mm}
%\end{table}
\end{wraptable}

\textbf{Training:} 
We optimize the cross-entropy loss from Eq.~\ref{eq:cross-entropy-loss}, where we replace $\phi$ for $\operatorname{sim}(q, a)$ and for each question we use wrong answers as negative samples.
We set the size of the matrices to $14 \times 14$ dimensions (equivalent to $105$ free parameters). %We conduct a grid search %to select optimal batch size and learning rate, 
%using the validation set. We report the average of $3$ runs. 
The set of hyperparameters can be found in Appendix~\ref{app:expdetails-qa}.

\textbf{Evaluation metrics:} At test time, for each question we rank its candidate answers according to Eq.~\ref{eq:qa-simfunction}.
We adopt MRR and H@1 as evaluation metrics.

\textbf{Baselines:} We compare against Euclidean and hyperbolic spaces of $105$ dimensions. For the Euclidean model we employ a linear layer as feature transformation. For the hyperbolic model, we operate on the tangent space and project the points into the Poincar\'e ball to compute distances.

\textbf{Results:} We present results in Table~\ref{tab:qa-results}. 
In both datasets, we see that the word embeddings and transformations learned by the SPD models are able to place questions and answers representations in the space such that they outperform Euclidean and hyperbolic baselines. Finsler metrics seem to be very effective in this scenario, improving the performance of their Riemannian counterparts in many cases.
%They offer the best performance in the three datasets.
% suggesting that clustering learns meaningful partitions of the input similarity graph.
Overall, this suggests that embeddings in SPD manifolds learn meaningful representations that can be exploited into downstream tasks. Moreover, we showcase how to employ different operations as feature transformations and bias additions, replicating the behavior of linear layers in classical deep learning architectures that can be seamlessly integrated with different distance metrics.
%map these embeddings to "flat" vectors; in this way they can be integrated with classical Euclidean network layers. 

\begin{figure*}[!t]
\vspace{-4mm}
\centering
\subfloat{\includegraphics[width=.33\textwidth,keepaspectratio,trim={0 0 0 11cm},clip]{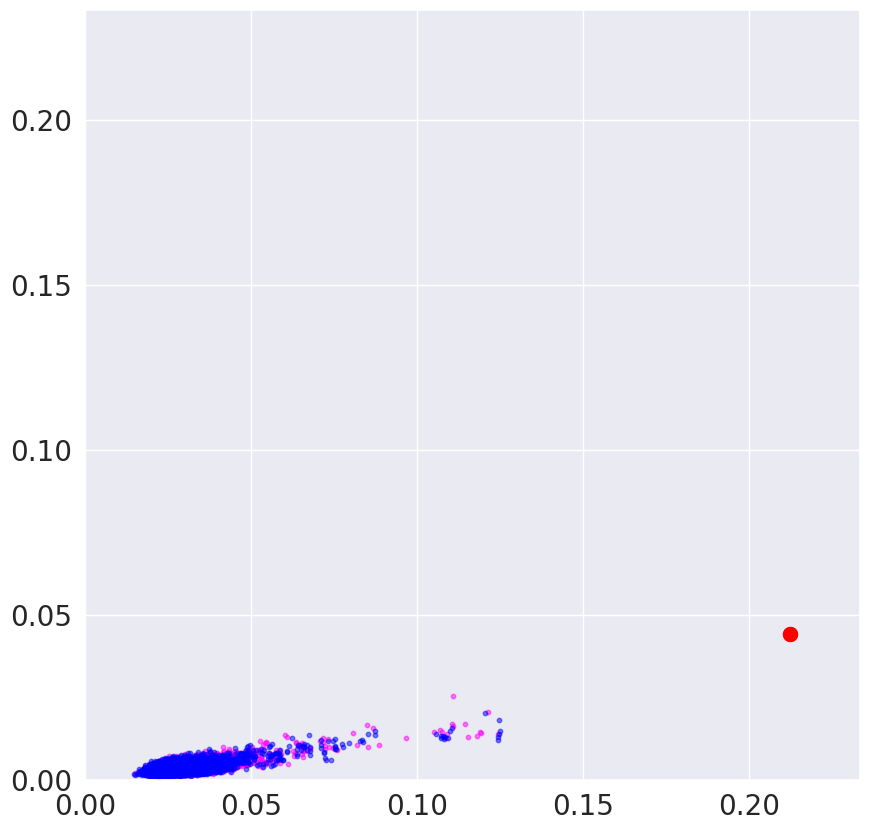}}\hfill
\subfloat{\includegraphics[width=.33\textwidth,keepaspectratio,trim={0 0 0 11cm},clip]{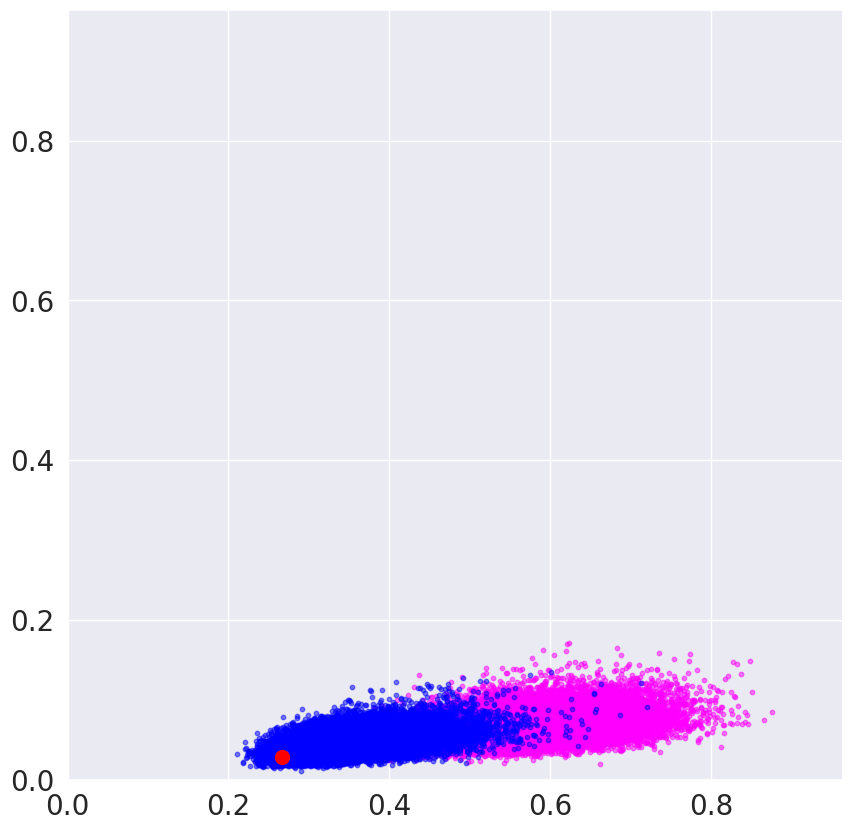}}\hfill
\subfloat{\includegraphics[width=.33\textwidth,keepaspectratio,trim={0 0 0 11cm},clip]{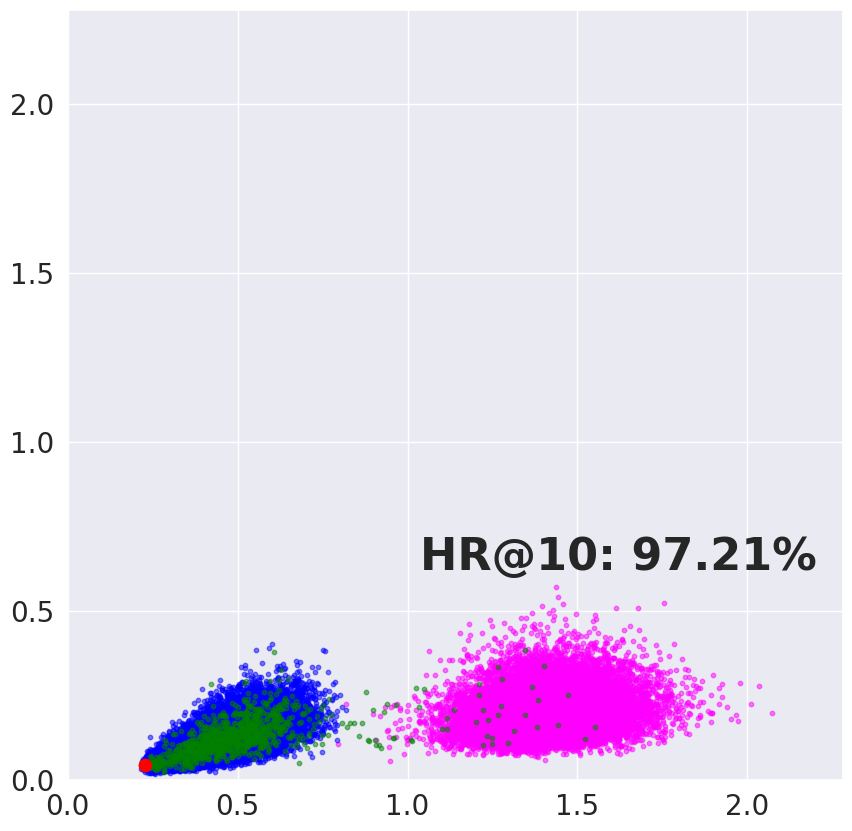}}\hfill \\
\subfloat{\includegraphics[width=.33\textwidth,keepaspectratio,trim={0 0 0 11cm},clip]{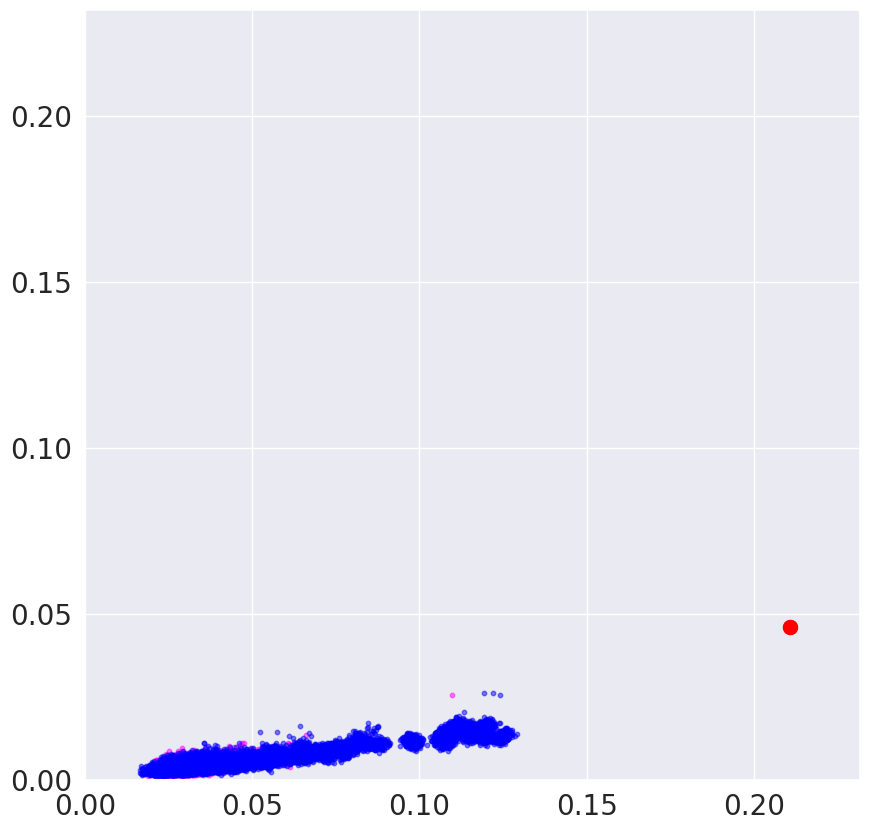}}\hfill
\subfloat{\includegraphics[width=.33\textwidth,keepaspectratio,trim={0 0 0 11cm},clip]{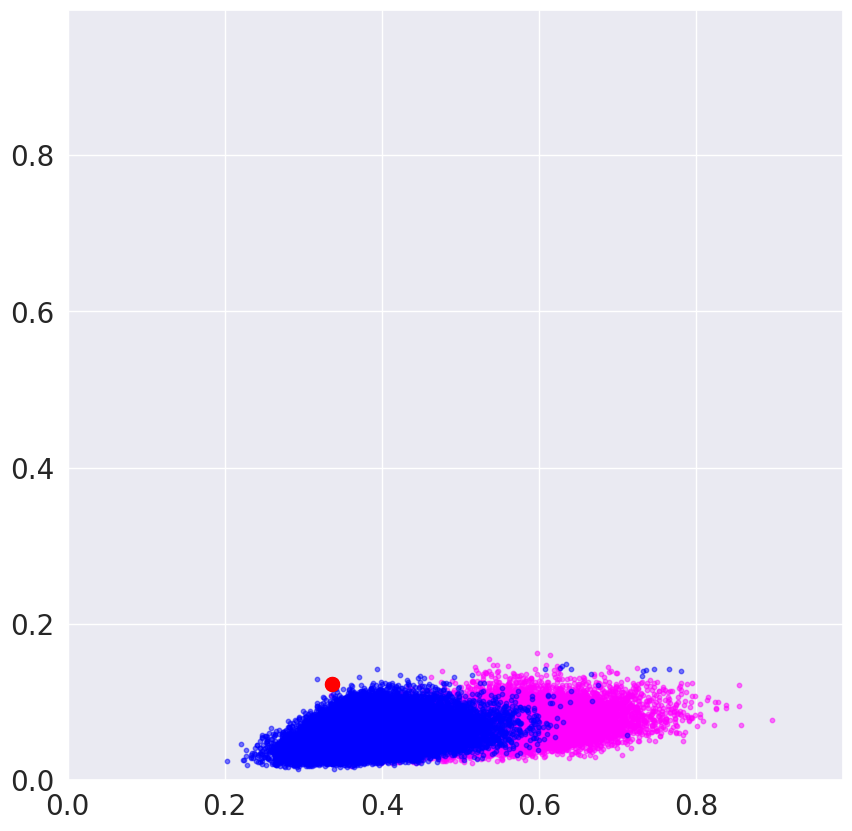}}\hfill
\subfloat{\includegraphics[width=.33\textwidth,keepaspectratio,trim={0 0 0 11cm},clip]{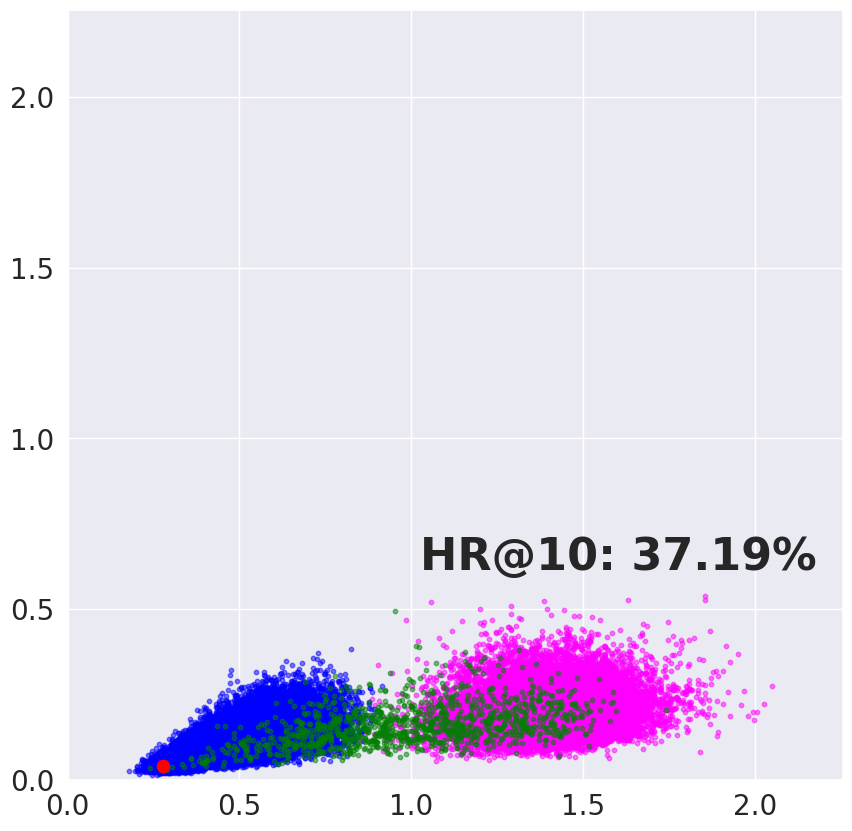}}\hfill
\caption{\textcolor{blue}{Train}, \textcolor{magenta}{negative} and \textcolor{ForestGreen}{validation} triples for relationships \textit{'derivationally related form'} (top) and \textit{'hypernym'} (bottom) for 5 (left), 50 (center) and 3000 (right) epochs for the \SPDtraFone model. The \textcolor{red}{red} dot corresponds to the relation addition $\textbf{R}$.}
\label{fig:normangle-plots}
\vspace{-4mm}
\end{figure*}

\subsection{Analysis}
\label{sec:analysis}
One reason to embed data into Riemannian manifolds, such as SPD, is to use geometric properties of the manifold to analyze the structure of the data \cite{lopez2021symmetric}. %Embeddings into hyperbolic spaces, for example, have been used to infer and visualize hierarchical structure in data sets \cite{nickel2018lorentz}.
Visualizations in SPD are difficult due to their high dimensionality. As a solution we use the vector-valued distance function to develop a new tool to visualize and analyze structural properties of the learned representations. 

%Given two matrices $A, B \in \spd_n$ the VVD is obtained by taking the logarithm of the eigenvalues of the matrix $A^{-1}B$ sorted in descending order.
We adopt the vector $(n-1, n-3, \cdots, -n+3, -n+1)$, as the barycenter of the space in $\RR^n$ where the VVD is contained. Then, we plot the norm of the VVD vector and its angle with respect to this barycenter.
In Figure~\ref{fig:normangle-plots}, we compute and plot the VVD corresponding to $d(\textbf{M}_r \otimes \textbf{H}, \textbf{T})$ and $\textbf{R}$ as defined in Eq.~\ref{eq:kg-scaling-scoring_fn} for KG models trained on WN18RR.
In early stages of the training, all points fall near the origin (left side of the plots). As training evolves, the model learns to separate true $(h, r, t)$ triples from corrupted ones (center part).
When the training converges, the model is able to clearly disentangle and cluster positive and negative samples. We observe how the position of the validation triples (green points, not seen during training) directly correlates with the performance of each relation. Plots for more relations in Appendix~\ref{app:expdetails-kgcompletion}.

\section{Conclusions}
\label{sec:conclusions}
Riemannian geometry has gained attention due to its capacity to represent non-Euclidean data arising in several domains.
In this work we introduce the vector-valued distance function, which allows to implement universal models (generalizing previous metrics on SPD), and can be leveraged to provide a geometric interpretation on what the models learn. Moreover, we bridge the gap between Euclidean and SPD geometry under the lens of the gyrovector theory, providing means to transfer standard arithmetic operations from the Euclidean setting to their analog notions in SPD. 
These tools enable practitioners to exploit the full representation power of SPD, and profit from the enhanced expressivity of this manifold.
We propose and evaluate SPD models on three tasks and eight datasets, which showcases the versatility of the approach and ease of integration with downstream tasks.
The results reflect the superior expressivity of SPD when compared to Euclidean or hyperbolic baselines. 
%We highlight that our approach is not task specific. 
%The VVD can be used in any task in which SPD has been applied before. The gyrocalculus we develop can be used to transfer any model based on standard arithmetic operations from the Euclidean setting to SPD. 

This work is not without limitations. We consider the computational complexity of working with spaces of matrices to be the main drawback, since the cost of many operations is polynomial instead of linear. Nevertheless, a matrix of rank $n$ implies $n(n+1)/2$ dimensions thus a large $n$ value is usually not required.

%means to translate basic operations into 
%We develop algorithms using the vector-valued distance and showcase two main advantages its versatility to implement universal models, and its use in explaining and visualizing what the model has learned.

%Furthermore, we bridge the gap between Euclidean and SPD geometry by developing gyrocalculus36in  SPD  (§4),  which  yields  closed-form  expressions  of  arithmetic  operations,  such  as  addition,37scalar multiplication and matrix scaling. This provides means to automatically translate previously38implemented ideas in different metric spaces to their analog notions in SPD.
%we introduce the vector valued distance function, which contains more information than any other metric on SPD
   % \item We develop SPD operations in the framework of gyrogroups
   
%- we propose scaling, rotations, reflections and additions in SPD, and express them in the framework of gyrocalculus. This allows to transfer many algorithms developed in Euclidean geometry seamlessly to SPD
    %\item We propose the Finsler metric on SPD
%- We showcase this transfer over 3 different tasks (2 KG + 4 RS + 2 QA = 8 datasets), obtaining improved performance in SPD. 
%- We leverage the vector valued distance to  provide a geometric interpretation on what the models are learning. 

\begin{ack}
This work has been supported by the German Research Foundation (DFG) as part of the Research Training Group Adaptive Preparation of Information from Heterogeneous Sources (AIPHES) under grant No. GRK 1994/1, as well as under Germany’s Excellence Strategy EXC-2181/1 - 390900948 (the Heidelberg STRUCTURES Cluster of Excellence), and by the Klaus Tschira Foundation, Heidelberg, Germany.
\end{ack}

\bibliographystyle{apalike}
\bibliography{mybib}

\newpage
\appendix
\section{Implementation Details}

\subsection{Computational Complexity of Operations}
\label{app:dist-alg-complexity}

In this section we discuss the computational theoretical complexity of the different operations involved in the development of this work. We employ Big O notation\footnote{\url{https://en.wikipedia.org/wiki/Big_O_notation}}. Since in all cases operations are not nested, but are applied sequentially, the costs can be added resulting in a polynomial expression. Thus, by applying the properties of the notation, we disregard lower-order terms of the polynomial.

\paragraph{Matrix Operations:}
For $n \times n$ matrices, the associated complexity of each operation is as follows:\footnote{\url{https://en.wikipedia.org/wiki/Computational_complexity_of_mathematical_operations}}
\begin{itemize}
    \item Addition and subtraction: $\bigO(n^2)$
    \item Multiplication: $\bigO(n^{2.4})$
    \item Inversion: $\bigO(n^{2.4})$
    \item Diagonalization: $\bigO(n^{3})$
\end{itemize}

\paragraph{SPD Operations:}
For $n \times n$ SPD matrices, the associated complexity of each operation is as follows:
\begin{itemize}
    \item Exp/Log map: $\bigO(n^{3})$, due to diagonalizations.
    \item Gyro-Addition: $\bigO(n^{2.4})$, due to matrix multiplications
    \item Matrix Scaling: $\bigO(n^{3})$, due to exp and log maps.
    \item Isometries: $\bigO(n^{2.4})$, due to matrix multiplications.
\end{itemize}

\paragraph{Distance Calculation:}
The full computation of the distance algorithm in $\spd_n$ involves matrix square root, inverses, multiplications, and diagonalizations. Since they are applied sequentially, without affecting the dimensionality of the matrices, we can take the highest value as the asymptotic cost of the algorithm, which is $\bigO(n^{3})$.

\subsection{Tangent Space Optimization}
\label{app:tg-space-optim}

Optimization in Riemannian manifolds normally requires Riemannian Stochastic Gradient Descent (RSGD) \cite{bonnabel2011rsgd} or other Riemannian techniques \cite{becigneul2019riemannianMethods}. We performed initial tests converting the Euclidean gradient into its Riemannian form, but found it to be less numerically stable and also slower than tangent space optimization \cite{chami2019hgcnn}.
With tangent space optimization, we can use standard Euclidean optimization techniques, and respect the geometry of the manifold. 
Note that tangent space optimization is an exact procedure, which does not incur losses in representational power. This is the case in $\spd_n$ specifically because of a completeness property given by the choice of $I \in \spd_n$ as the basepoint: there is always a global bijection between the tangent space and the manifold.

\section{Experimental Details}
\label{sec:exp-details}
All models and experiments were implemented in PyTorch \cite{paszke2019pytorchNeurips} with distributed data parallelism, for high performance on clusters of CPUs/GPUs.

\paragraph{Hardware:}
All experiments were run on Intel Cascade Lake CPUs, with microprocessors Intel Xeon Gold 6230 (20 Cores, 40 Threads, 2.1 GHz, 28MB Cache, 125W TDP). 
Although the code supports GPUs, we did not utilize them due to higher availability of CPU's.

\subsection{Knowledge Graph Completion}
\label{app:expdetails-kgcompletion}

\paragraph{Setup:} We train for $5000$ epochs, with batch size of $4096$ and $10$ negative samples, reducing the learning rate by a factor of $2$ if the model does not improve the performance on the dev set after $50$ epochs, and early stopping based on the MRR if the model does not improve after $500$ epochs. We use the burn-in strategy \cite{nickel2017poincare} training with a 10 times smaller learning rate for the first 10 epochs.
We experiment with matrices of dimension $n \times n$ where $n \in \{14, 20, 24\}$ (this is the equivalent of $\{105, 210, 300\}$ degrees of freedom respectively), learning rates from $\{1\rm{e-}4, 5\rm{e-}5, 1\rm{e-}5\}$ and weight decays of $\{1\rm{e-}2, 1\rm{e-}3\}$.

\begin{table}[!b]
\vspace{-1mm}
\caption{Statistics for Knowledge Graph completion datasets.}
\label{tab:kg-datastats}
\small
\centering
\adjustbox{max width=\linewidth}{
\begin{tabular}{lrrrrr}
\toprule
\textbf{Dataset} & \multicolumn{1}{c}{\textbf{Entities}} & \multicolumn{1}{c}{\textbf{Relations}} & \multicolumn{1}{c}{\textbf{Train}} & \multicolumn{1}{c}{\textbf{Dev}} & \multicolumn{1}{c}{\textbf{Test}} \\
\midrule
WN18RR & 40943 & 11 & 86835 & 3034 & 3134 \\
FB15k-237 & 14541 & 237 & 272115 & 17535 & 20466 \\
\bottomrule
\end{tabular}
}
\end{table}

\paragraph{Datasets:} Stats about the datasets used in Knowledge graph experiments can be found in Table~\ref{tab:kg-datastats}.

\paragraph{Results:} In addition to the results provided in \S\ref{sec:kg-completion}, in Table~\ref{tab:kg-results-appendix} we provide a comparison with other state-of-the-art models. We include ComplEx \cite{trouillon2016complex}, Tucker \cite{balazevic2019tucker}, and Quaternion \cite{zhang2019quaternionKG}. 

\paragraph{Analisis:} In Figure~\ref{fig:appendix-normangle-plots} we add equivalent plots to the ones explained in \S\ref{sec:analysis} for other relations from WN18RR.

\begin{figure*}[!t]
%\vspace{-4mm}
\centering
\subfloat[\textit{'Also see'}]{\includegraphics[width=.33\textwidth,keepaspectratio,trim={0 0 0 11cm},clip]{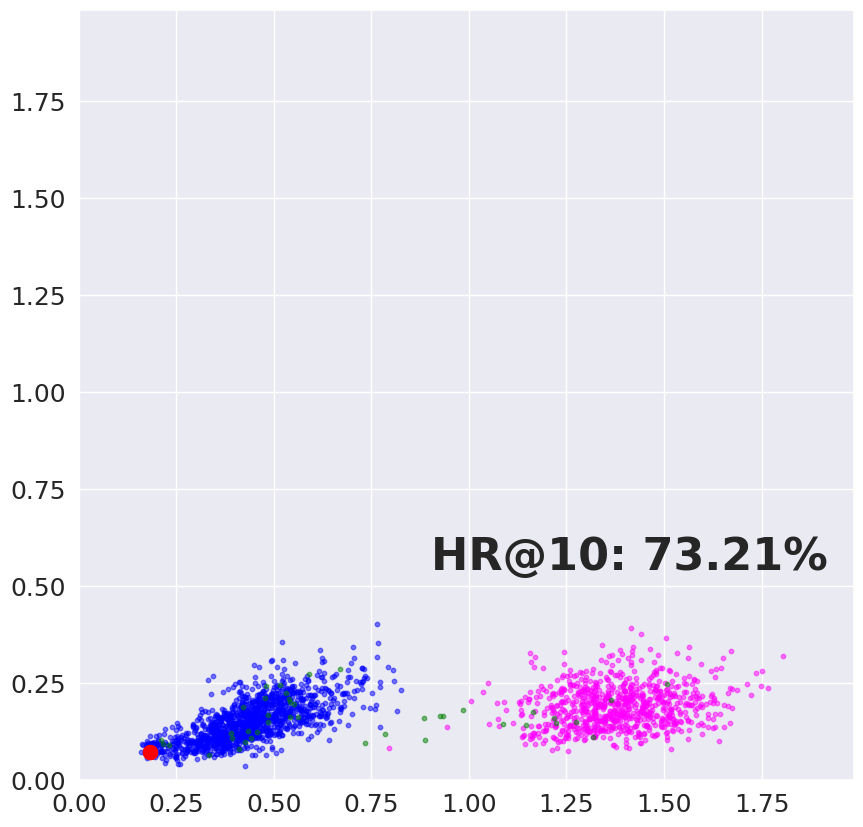}}\hfill
\subfloat[\textit{'Has part'}]{\includegraphics[width=.33\textwidth,keepaspectratio,trim={0 0 0 11cm},clip]{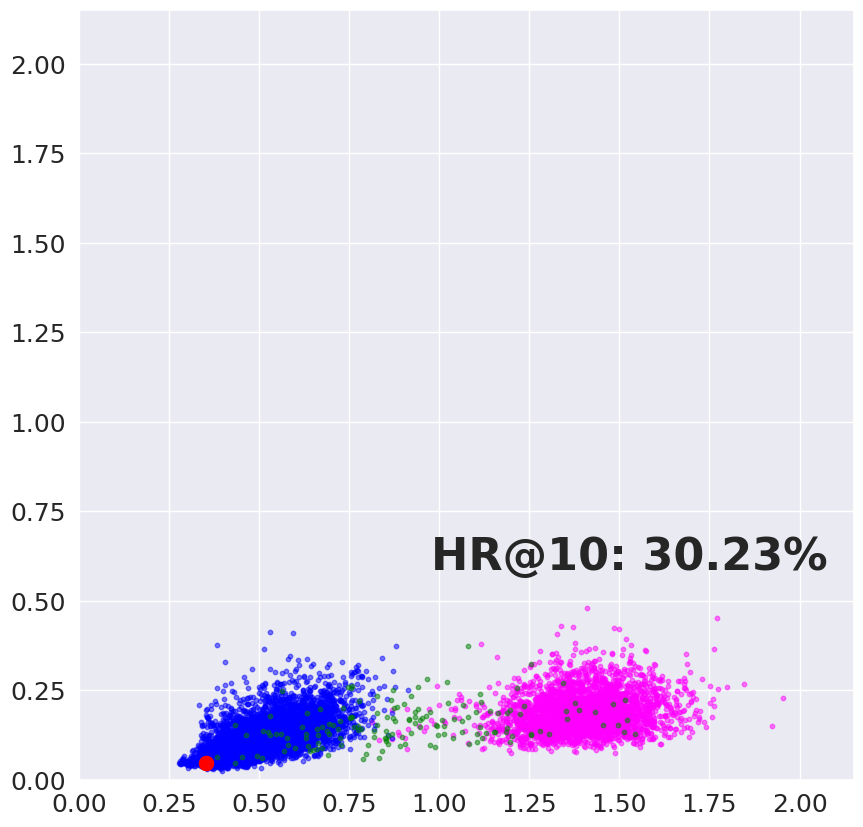}}\hfill
\subfloat[\textit{'Instance hypernym'}]{\includegraphics[width=.33\textwidth,keepaspectratio,trim={0 0 0 11cm},clip]{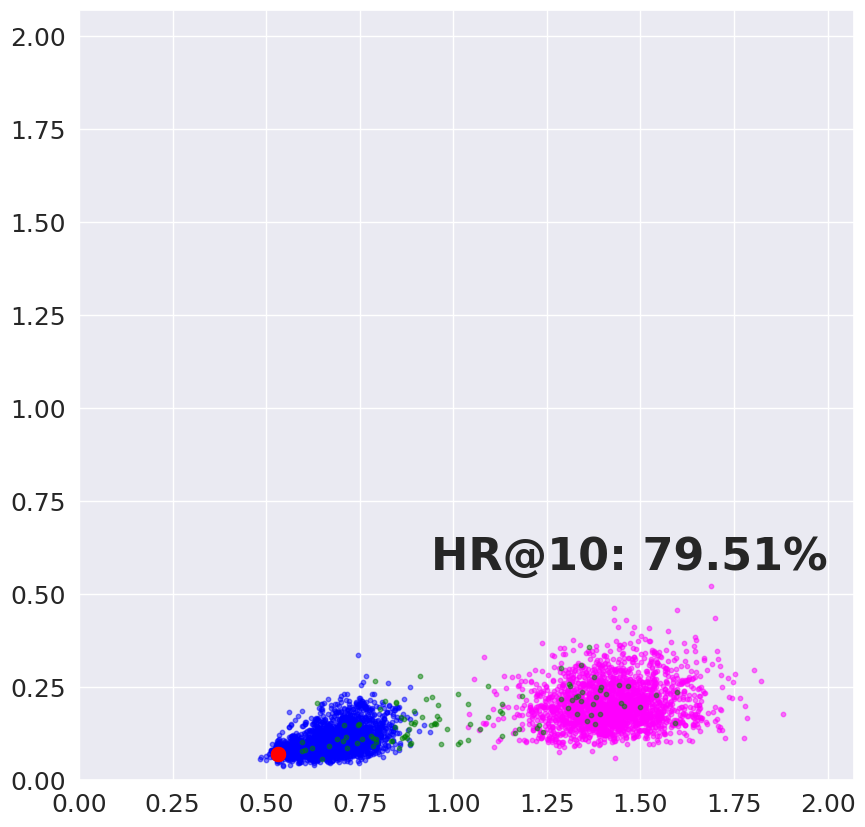}}\hfill \\
\subfloat[\textit{'Member meronym'}]{\includegraphics[width=.33\textwidth,keepaspectratio,trim={0 0 0 11cm},clip]{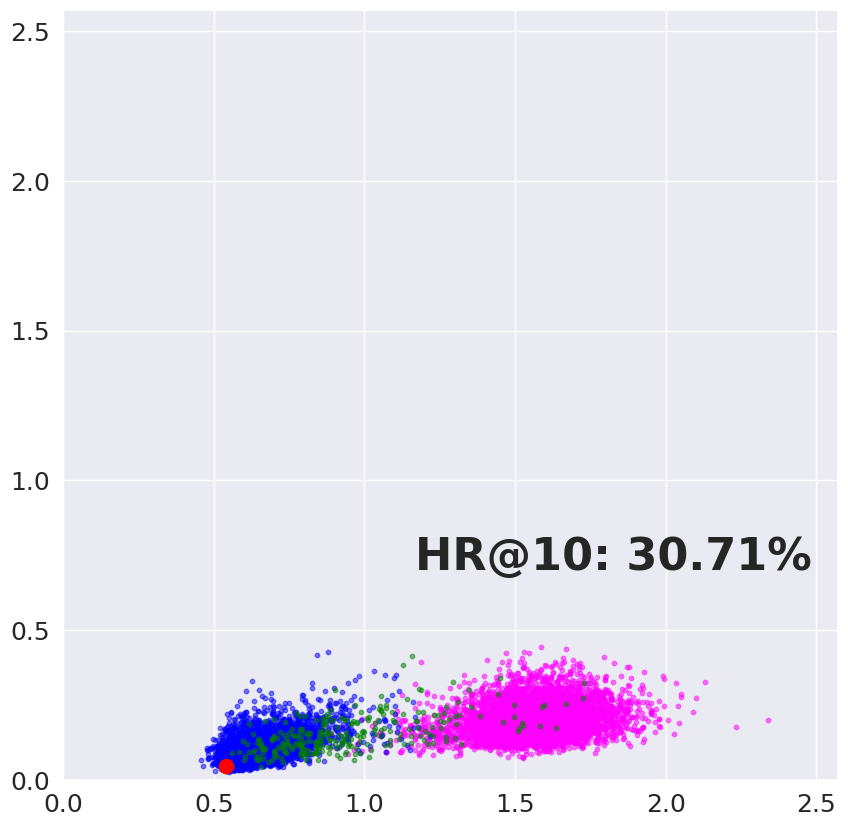}}\hfill
\subfloat[\textit{'Member of domain region'}]{\includegraphics[width=.33\textwidth,keepaspectratio,trim={0 0 0 11cm},clip]{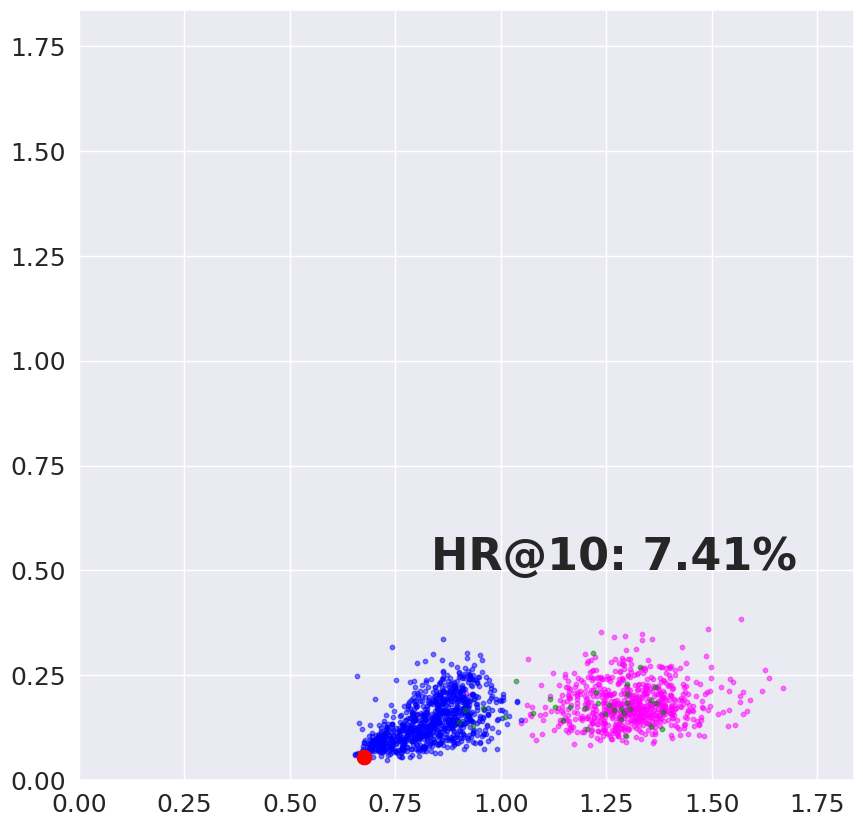}}\hfill
\subfloat[\textit{'Member of domain usage'}]{\includegraphics[width=.33\textwidth,keepaspectratio,trim={0 0 0 11cm},clip]{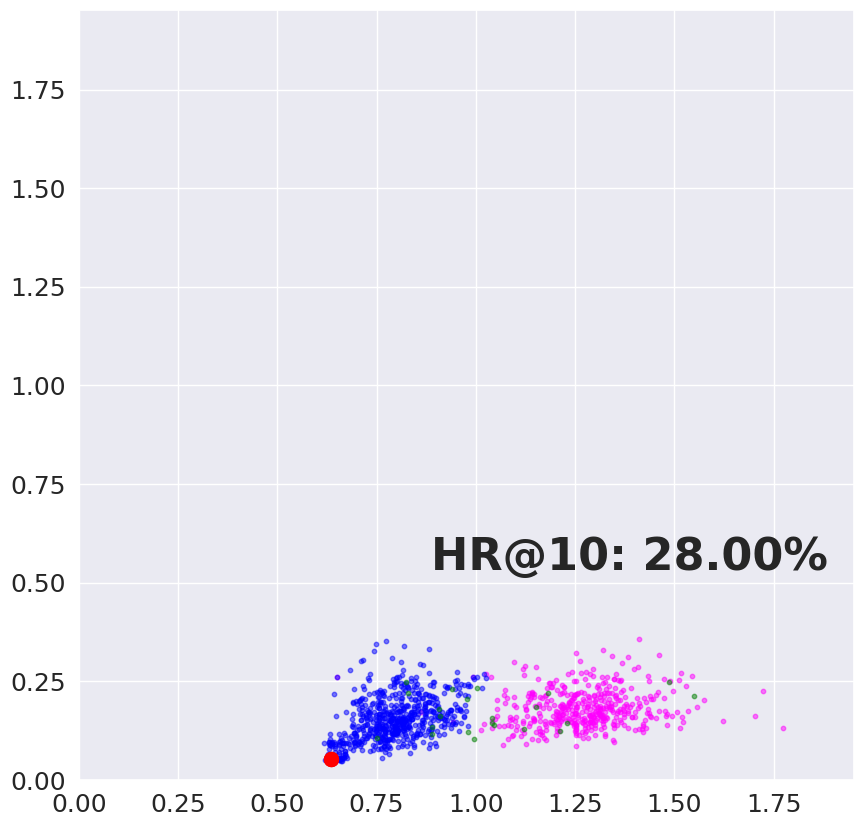}}\hfill \\
\subfloat[\textit{'Similar to'}]{\includegraphics[width=.33\textwidth,keepaspectratio,trim={0 0 0 11cm},clip]{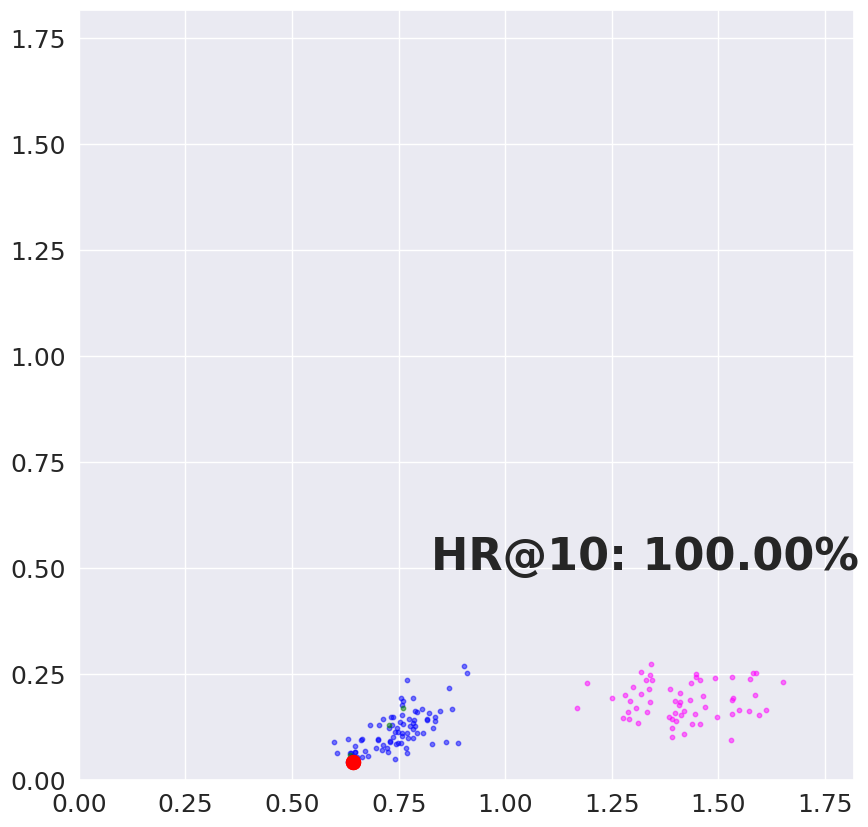}}\hfill
\subfloat[\textit{'Synset domain topic of'}]{\includegraphics[width=.33\textwidth,keepaspectratio,trim={0 0 0 11cm},clip]{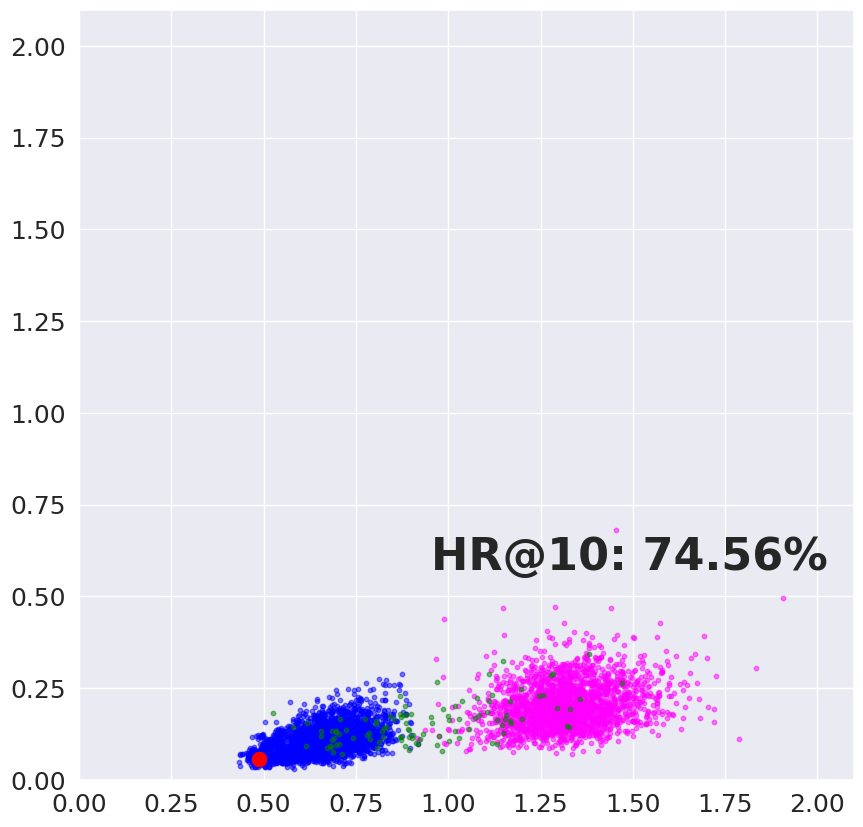}}\hfill
\subfloat[\textit{'Verb group'}]{\includegraphics[width=.33\textwidth,keepaspectratio,trim={0 0 0 11cm},clip]{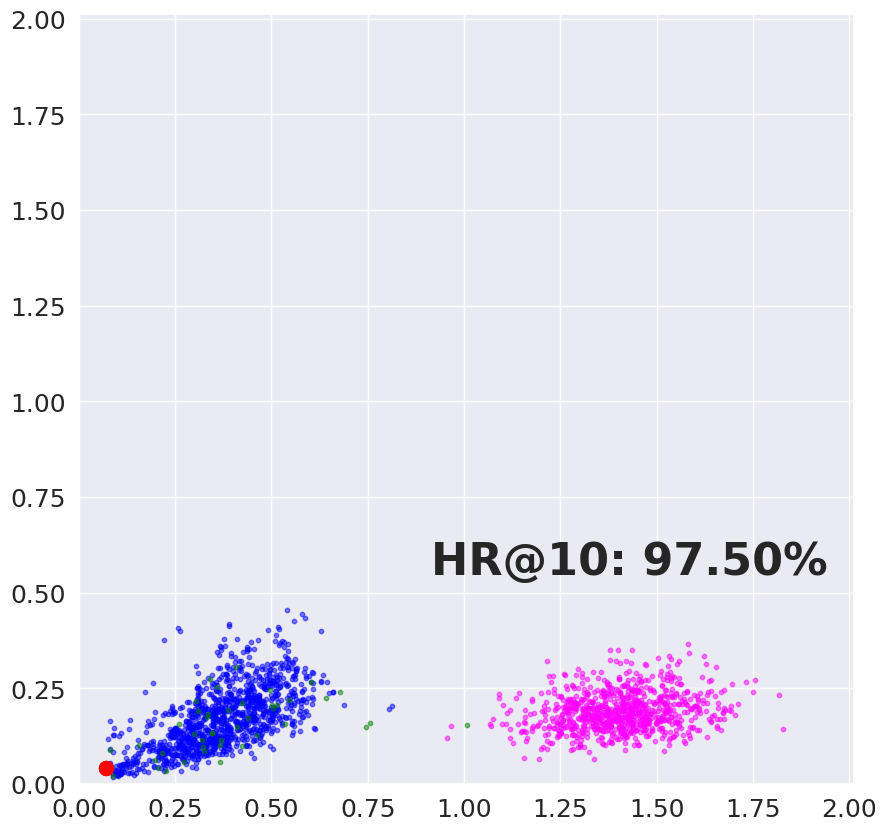}}\hfill
\caption{\textcolor{blue}{Train}, \textcolor{magenta}{negative} and \textcolor{ForestGreen}{validation} triples for WN18RR relationships of the \SPDtraFone model after convergence. The \textcolor{red}{red} dot corresponds to the relation addition $\textbf{R}$.}
\label{fig:appendix-normangle-plots}
%\vspace{-4mm}
\end{figure*}

\begin{table}[!t]
\vspace{-1mm}
\caption{Results for Knowledge graph completion.}
\label{tab:kg-results-appendix}
\small
\centering
\adjustbox{max width=\linewidth}{
\begin{tabular}{llrrrrrrrr}
\toprule
 &  & \multicolumn{4}{c}{WN18RR} & \multicolumn{4}{c}{FB15k-237} \\
 \cmidrule(lr){3-6}\cmidrule(lr){7-10}
Space & Model & \multicolumn{1}{l}{MRR} & \multicolumn{1}{l}{HR@1} & \multicolumn{1}{l}{HR@3} & \multicolumn{1}{l}{HR@10} & \multicolumn{1}{l}{MRR} & \multicolumn{1}{l}{HR@1} & \multicolumn{1}{l}{HR@3} & \multicolumn{1}{l}{HR@10} \\
\midrule
\multirow{2}{*}{$\mathbb C$} & ComplEx & 48.0 & 43.5 & 49.5 & 57.2 & 35.7 & 26.4 & 39.2 & 54.7 \\
 & RotC & 47.6 & 42.8 & 49.2 & 57.1 & 33.8 & 24.1 & 37.5 & 53.3 \\
 \hline
$\mathbb Q$ & Quaternion & 48.8 & 43.8 & 50.8 & 58.2 & \textbf{36.6} & \textbf{27.1} & \textbf{40.1} & \textbf{55.6} \\
\hline
\multirow{4}{*}{$\RR$} & Tucker & 47.0 & 44.3 & 48.2 & 52.6 & 35.8 & 26.6 & 39.4 & 54.4 \\
 & MuRE & 47.5 & 43.6 & 48.7 & 55.4 & 33.6 & 24.5 & 37.0 & 52.1 \\
 & RotE & 49.4 & 44.6 & 51.2 & 58.5 & 34.6 & 25.1 & 38.1 & 53.8 \\
 & RefE & 47.3 & 43.0 & 48.5 & 56.1 & 35.1 & 25.6 & 39.0 & 54.1 \\
 \hline
\multirow{3}{*}{$\mathbb H$} & MuRP & 48.1 & 44.0 & 49.5 & 56.6 & 33.5 & 24.3 & 36.7 & 51.8 \\
 & RotH & \textbf{49.6} & \textbf{44.9} & \textbf{51.4} & 58.6 & 34.4 & 24.6 & 38.0 & 53.5 \\
 & RefH & 46.1 & 40.4 & 48.5 & 56.8 & 34.6 & 25.2 & 38.3 & 53.6 \\
 \hline
\multirow{6}{*}{$\spd$} & \SPDtraRiem & 48.1 & 43.1 & 50.1 & 57.6 & 34.5 & 25.1 & 38.0 & 53.5 \\
 & \SPDtraFone & 48.4 & 42.6 & 51.0 & \textbf{59.0} & 32.9 & 23.6 & 36.3 & 51.5 \\
 & \SPDrotRiem & 46.2 & 39.7 & 49.6 & 57.8 & 32.7 & 23.4 & 36.1 & 51.4 \\
 & \SPDrotFone & 38.6 & 25.6 & 48.7 & 56.8 & 31.4 & 22.3 & 34.7 & 49.8 \\
 & \SPDrefRiem & 48.3 & 44.0 & 49.7 & 56.7 & 32.5 & 23.4 & 35.6 & 51.0 \\
 & \SPDrefFone & 48.7 & 44.3 & 50.1 & 57.4 & 30.7 & 21.7 & 33.7 & 48.8 \\
 \bottomrule
\end{tabular}
}
\vspace{-4mm}
\end{table}

\begin{table}[!b]
\vspace{-1mm}
\caption{Statistics for KG Recommender datasets.}
\label{tab:recosys-datastats}
\small
\centering
\adjustbox{max width=\linewidth}{
\begin{tabular}{lrrrrrr}
\toprule
\multirow{2}{*}{\textbf{Dataset}} & \multirow{2}{*}{\textbf{Users}} & \multirow{2}{*}{\textbf{Items}} & \multirow{2}{*}{\textbf{\begin{tabular}[c]{@{}r@{}}Other \\ Entities\end{tabular}}} & \multicolumn{2}{c}{\textbf{Train Relations}} & \multirow{2}{*}{\textbf{Dev/Test}} \\
\cmidrule(lr){5-6}
 &  &  &  & \textbf{User-item} & \textbf{Others} &  \\
\midrule
Software & 1826 & 802 & 689 & 8242 & 6078 & 1821 \\
Luxury Beauty & 3819 & 1581 & 2 & 20796 & 26044 & 3639 \\
%Industrial/Sci & 11041 & 5334 & 2920 & 50135 & 63690 & 10945 \\
Prime Pantry & 14180 & 4970 & 1100 & 102848 & 99118 & 14133 \\
%Digital Music & 16566 & 11797 & 105 & 112278 & 269 & 16507 \\
%Musical Instruments & 27530 & 10620 & 2235 & 164206 & 118448 & 27475 \\
MindReader & 961 & 2128 & 11775 & 11279 & 99486 & 953 \\
\bottomrule
\end{tabular}
}
\end{table}

\subsection{Knowledge Graph Recommender Systems}
\label{app:expdetails-kgrecosys}

\paragraph{Setup:} We train for $3000$ epochs, with batch size from $\{512, 1024\}$ and $10$ negative samples, and early stopping based on the MRR if the model does not improve after $200$ epochs. We use the burn-in strategy \cite{nickel2017poincare} training with a 10 times smaller learning rate for the first 10 epochs.
We report average $\pm$ standard deviation of 3 runs.
We experiment with matrices of dimension $10 \times 10$ (this is the equivalent of $\{55\}$ degrees of freedom), learning rates from $\{5\rm{e-}4, 1\rm{e-}4, 5\rm{e-}5\}$ and weight decay of $1\rm{e-}3$. Same grid search is applied to baselines.

\paragraph{Datasets:} On the Amazon dataset we adopt the 5-core split for the branches "Software", "Luxury \& Beauty" and "Industrial \& Scientific", which form a diverse dataset in size and domain.
We add relationships used in previous work \cite{zhang2018learningOverKBE, ai2018explreco}. These are:
\begin{itemize}
    \item \textit{also\_bought}: %triples $(item\_A, also\_bought, item\_B)$, meaning that 
    users who bought item A also bought item B. % (complementary products, according to \cite{mcauley2015complementarySuplementary}).
    \item \textit{also\_view}: %triples $(item\_A, also\_view, item\_B)$ meaning that 
    users who bought item A also viewed item B.
    %(supplementary products).
    \item \textit{category}: %triples $(item,$ $has\_category,$ $category\_A)$. Item metadata consist of categorical labels, which means that 
    the item belongs to one or more categories.%, without particular order.
    \item \textit{brand}: %triples $(item, has\_brand, brand\_A)$. Item metadata contains 
    the item belongs to one brand. 
\end{itemize}
On the MindReader dataset, we consider a user-item interaction when a user gave an explicit positive rating to the movie. The relationships added are:
\begin{itemize}
    \item \textit{directed\_by}: the movie was directed by this person.
    \item \textit{produced\_by}: the movie was produced by this person/company.
    \item \textit{from\_decade}: the movie was released in this decade.
    \item \textit{followed\_by}: the movie was followed by this other movie.
    \item \textit{has\_genre}: the movie belongs to this genre.
    \item \textit{has\_subject}: the movie has this subject.
    \item \textit{starring}: the movie was starred by this person.
\end{itemize}
Statistics of the datasets with the added relationships can be seen in Table~\ref{tab:recosys-datastats}. For dev/test we only consider users with 3 or more interactions.

\subsection{Question Answering}
\label{app:expdetails-qa}

\paragraph{Setup:} We train for $300$ epochs, with $2$ negative samples and early stopping based on the MRR if the model does not improve after $20$ epochs. We use the burn-in strategy \cite{nickel2017poincare} training with a 10 times smaller learning rate for the first 10 epochs. We report average $\pm$ standard deviation of 3 runs.
We experiment with matrices of dimension $14 \times 14$ (equivalent of $105$ degrees of freedom respectively), batch size from $\{512, 1024\}$, learning rate from $\{1\rm{e-}4, 5\rm{e-}5, 1\rm{e-}5\}$ and weight decays of $1\rm{e-}3$.
Same grid search was applied to baselines.

\paragraph{Datasets:} Stats about the datasets used for Question Answering experiments can be found in Table~\ref{tab:qa-datastats}.

\begin{table}[!t]
\caption{Statistics for Question Answering datasets.}
\label{tab:qa-datastats}
\small
\centering
\adjustbox{max width=\linewidth}{
\begin{tabular}{lrr}
\toprule
 & \multicolumn{1}{c}{\textsc{TrecQA}} & \multicolumn{1}{l}{\textsc{WikiQA}} \\
 \midrule
Train Qs & 1227 & 2119 \\
Dev Qs & 65 & 127 \\
Test Qs & 68 & 244 \\
\midrule
Train pairs & 53417 & 20361 \\
Dev Pairs & 1117 & 1131 \\
Test Pairs & 1442 & 2352 \\
\bottomrule
\end{tabular}
}
\end{table}

\section{Differential Geometry of $\spd_n$}
\label{sec:DiffGeo_SPD}

\subsection{Orthogonal Diagonalization}
\label{s.orthdiag}
Every real symmetric matrix may be orthogonally diagonalized: For every point $P\in\spd_n$ we may find a positive diagonal matrix $D$ and an orthogonal matrix $K$ such that $P=KDK^T$.  
This diagonalization has two practical consequences: it allows  efficient computation of important $\spd_n$ operations, and provides another means of generalizing Euclidean notions to $\spd_n$.

With respect to computation, if $P\in\spd_n$ has orthogonal diagonalization $P=KDK^T$, we may compute its square root and logarithm as
$\sqrt{P}=K\sqrt{D}K^T$ and $\log(P)=K\log(D)K^T$
where $\sqrt{D}=\operatorname{diag}(\sqrt{d_1},\ldots,\sqrt{d_n})$ and $\log(D)=\operatorname{diag}(\log{d_1},\ldots,\log{d_n})$ for $D=\operatorname{diag}(d_1,\ldots,d_n)$.  
Similarly, if a tangent vector $U\in S_n$ has orthogonal diagonalization $U=K\Lambda K^T$ (here $\Lambda=\operatorname{diag}(\lambda_1,\ldots,\lambda_n)$ not necessarily positive definite), the exponential map is computed as
$\exp(U)=K e^\Lambda K^T$,
where $e^\Lambda=\operatorname{diag}(e^{\lambda_1},\ldots e^{\lambda_n})$.

We verify this in the two lemmas below.

\begin{lemma}
If $K\in O(n)$ and $X$ is any $n\times n$ matrix, then $\exp(KXK^T)=K\exp(X)K^T$.  
\label{prop:MatrixExp-Conj}
\end{lemma}
\begin{proof}
As $K$ is orthogonal, $K^T=K^{-1}.$
Conjugation is an automorphism of the algebra of $n\times n$ matrices, and so applying this to any partial sum of the exponential $\exp(X)=\sum_{n=0}^\infty \frac{1}{n!}X^n$ yields
$$\sum_{n=0}^N\frac{1}{n!}(KXK^{-1})^n=K\left(\sum_{n=0}^N\frac{1}{n!}X^n\right)K^{-1}.$$
Taking the limit of this equality as $N\to\infty$ gives the claimed result.
\end{proof}

\begin{lemma}
If $D=\operatorname{diag}(d_1,\ldots, d_n)$ is a diagonal matrix, then $\exp(D)=\operatorname{diag}(e^{d_1},\ldots, e^{d_n})$. 
\label{prop:MatrixExp-Diag}
\end{lemma}
\begin{proof}
The multiplication of diagonal matrices coincides with the elementwise product of their diagonal entries.  Again applying this to any partial sum of the exponential of $D=\operatorname{diag}(d_1,\ldots,d_n)$ gives
$$\sum_{n=0}^N\frac{1}{n!}\operatorname{diag}(\ldots,d_i\ldots)^n=\operatorname{diag}\left(\ldots,\sum_{n=0}^N\frac{1}{n!}d_i^n,\ldots\right).$$
Taking the limit of this equality as $N\to\infty$ gives the claimed result.
\end{proof}

\subsection{Metric and Isometries}
\label{sec:metric_isom}
The Riemannian metric on $\spd_n$ is defined as follows: 
if $U,V\in S_n$ are tangent vectors based at $P\in \spd_n$, their inner product is:
$$\langle U,V\rangle_P=\tr(P^{-1}UP^{-1}V).$$ 
Note that at the basepoint, this is just the standard matrix inner product $\langle U,V\rangle_I=\tr(UV^T)$ as $U,V$ are symmetric.
We now verify the $GL(n,\RR)$ action given by $M$ acting as $P\mapsto MPM^T$ is an action by isometries of this metric.

\begin{lemma}
The action $f(P)=MPM^T$ extends to tangent vectors $U$ based at $P$ without change in formula: $f_\ast(U)=MUM^T$
\label{prop:Isometry-TangVector}
\end{lemma}
\begin{proof}
Let $P\in\spd_n$ and $U\in S_n$ be a tangent vector based at $P$.  Then by definition, $U=P^\prime _0$ is the derivative of some path $P_t$ of some path of matrices in $\spd_n$ throguh $P_0=P$. 
We compute the action of $P\to MPM^T$ on $U$ by taking the derivative of its action on the path:
$$\frac{d}{dt}\Big|_{t=0}MP_tM=MP^\prime_t M\Big|_{t=0}=MUM^T$$
\end{proof}

\begin{proposition}
For every $M\in GL(n;\RR)$ the transformation $M\mapsto MPM^T$ preserves the Riemannian metric on $\spd_n$.
\label{prop:Isometry}
\end{proposition}
\begin{proof}
Let $M\in GL(n;\RR)$ and choose arbitrary point $P\in \spd_n$, and tangent vectors $U,V\in T_P\spd_n$.  We compute the pullback of the metric under the symmetry $f(P)=MPM^T$.
Computing directly from the definition an the previous lemma,
\begin{equation*}
    \begin{split}
        f^\ast\langle U,V\rangle_P&=\langle f_\ast U,f_\ast V\rangle_{f(P)}\\
        &=\langle MUM^T, MVM^T\rangle_{MPM^T}\\
        &=\tr\left(\left(MPM^T\right)^{-1}MUM^T\left(MPM^T\right)^{-1}MVM^T\right)\\
        &=\tr\left(M^{-T}P^{-1}UP^{-1}VM^T\right)\\
        &=\tr\left(P^{-1}UP^{-1}V\right)\\
        &=\langle U,V\rangle_P,
    \end{split}
\end{equation*}
where the penultimate equality uses that trace is invariant under conjugacy.

\end{proof}

This provides a vivid geometric interpretation of the previously discussed orthogonal diagonalization operation on $\spd_n$.

\begin{corollary}
Given any $P\in\spd_n$, there exists a symmetry fixing $I$ which moves $P$ to a diagonal matrix.
\label{prop:Isometry_MaxFlat}
\end{corollary}

This subspace of diagonal matrices plays an essential role in working with $\spd_n$.  As we verify below, the intrinsic geometry of this subspace of diagonal matrices inherited from the Riemannian metric on $\spd_n$ is flat.

\begin{proposition}
Let $\mathcal{D}\subset\spd_n$ be the set of diagonal matrices, and define $f\colon\RR^n\to \mathcal{D}$ by $f(x_1,\ldots,x_n)=\operatorname{diag}(e^{x_1},\ldots, e^{x_n})$.  Then $f$ is an isometry from the Euclidean metric on $\RR^n$ to the metric on $\mathcal{D}$ induced from $\spd_n$.
\label{prop:MaxFlat_Euclidean}
\end{proposition}
\begin{proof}
We pull back the metric on $\mathcal{D}$ by $f$, and see that on $\RR^n$ this results in the standard Euclidean metric.
Given a point $x\in\RR^n$ with tangent vectors $y,z\in\RR^n$, we compute this as
$$f^\ast\langle y,z\rangle_x=\langle f_\ast y,f_\ast z\rangle_{f(x)}$$
From the definition of $f$, we see that the pushforward of $y$ along $f$ is $\operatorname{diag}(\ldots,e^{x_i} y_i,\ldots)$ and similarly for $z$.
Thus we may compute directly and see the result is the standard dot product on $\RR^n$.
\begin{equation*}
    \begin{split}
        \langle \operatorname{diag}(e^{x_i} y_i),\operatorname{diag}(e^{x_i} z_i)\rangle_{\operatorname{diag}(e^{x_i})}&=\tr\left(\operatorname{diag}(e^{x_i} y_i)\operatorname{diag}(e^{-x_i})\operatorname{diag}(e^{x_i} y_i)\operatorname{diag}(e^{-x_i})\right)\\
       % &=\tr\left(\operatorname{diag}\left(\frac{e^{x_i}y_ie^{x_i}z_i}{e^{-2x_i}}\right)\right)\\
        &=\tr\left(\operatorname{diag}(y_iz_i)\right)\\
        &=\sum_{i=1}^n y_iz_i
    \end{split}
\end{equation*}
\end{proof}

This subspace $\mathcal{D}$ is in fact a \emph{maximal flat} for $\spd_n$, the largest dimensional totally geodesic Euclicean submanifold embedded in $\spd_n$.  For more information on the general theory of symmetric spaces from which the notion of maximal flats arises, see Helgason \cite{helgason1078diffGeom}.  For our purposes, it is only important to note the following fact.

\begin{corollary}
The set of diagonal matrices in $\spd_n$ is an isometrically and totally geodesically embedded copy of euclidean $n$-space.
\label{prop:MaxFlat_TotGeodesic}
\end{corollary}

\subsection{Exponential and Logarithmic Maps}
\label{sec:exponential-maps}

The Riemannian exponential map gives
a connection between the Euclidean geometry of the tangent space $S_n$ and the curved geometry of $\spd_n$. It assigns the tangent vector $U$ to the point $Q=\exp(U)$ of $\spd_n$ reached by traveling along the geodesic starting from the basepoint $I$ in direction $U$ for distance $\|U\|$.

As a consequence of non-positive curvature, $\exp$ is a diffeomorphism of $S_n$ onto $\spd_n$, and so has an inverse: the Riemannian logarithm $\log\colon \spd_n\to S_n$.  See \cite{eberlein1985} for a review of the general theory of manifolds of non-positive curvature.
Together, this pair of functions allows one to freely move between ’tangent space coordinates’ or the original ’manifold coordinates' which we exploit to transfer Euclidean optimization schemes to $\spd_n$ (see \S\ref{sec:optimization}).

Secondly, the geometry of $\spd_n$ is so tightly tied to the algebra of $n\times n$ matrices that the Riemannian exponential agrees exactly with the usual matrix exponential, and the Riemannian logarithm is the matrix logarithm (because of this, we do not distinguish the two notationally), as we verify in the proposition below.
Both of these are readily computable via orthogonal diagonalization, as given in \S\ref{s.orthdiag}.
This is in stark contrast to general Riemannian manifolds, where the exponential map may have no simple formula.

\begin{proposition}
Let $\exp_{\textrm{Riem}}\colon S_n\to \spd_n$ be the Riemannian exponential map based at $I\in \spd_n$, and $\exp$ be the matrix exponential.  Then $\exp_{\textrm{Riem}}=\exp$.
\label{prop:Exp_RiemVsMatrix}
\end{proposition}
\begin{proof}
Let $U\in S_n$ be a tangent vector to $\spd_n$ at the basepoint $I$, and orthogonally diagonalize as $U=KDK^T$ for some $K\in O(n)$, $D=\operatorname{diag}(d_1,\ldots,d_n)$.  As $D$ is tangent to the maximal flat $\mathcal{D}$ of diagonal matrices, the geodesic segment $\exp_{Riem}(tD)$ must be a geodesic in $\mathcal{D},$ which we know from Lemma \ref{prop:Exp_RiemVsMatrix} to be the coordinate-wise exponential of a straight line in $\RR^n$.
Precisely, this geodesic is $\operatorname{diag}(\ldots, e^{d_i t},\ldots)$, and so the original geodesic with initial tangent $U=KDK^T$ is
$\exp_{Riem}(tU)=K\operatorname{diag}(\ldots, e^{d_i t},\ldots)K^T$ by Lemma \ref{prop:MatrixExp-Conj}.
Specializing to $t=1$, this gives the claim:
\begin{equation*}
    \begin{split}
     \exp_{Riem}(U)&=K\exp_{Riem}(D)K^T\\
     &=K\diag(\ldots, e^{d_i},\ldots)K^{-1}\\
     &=K\exp(D)K^{-1}\\
     &=\exp(KDK^{-1})\\
     &=\exp(U)  
    \end{split}
\end{equation*}
\end{proof}

This easily transfers to an understanding of the Riemannian exponential at an arbitrary point $P\in \spd_n$, if we identify the tangent space at $P$ with the symmetric matrices $S_n$ as well.

\begin{corollary}
The exponential based at an arbitrary point $P\in\spd_n$ is given by 
$$\exp_{Riem,P}(U)=\sqrt{P}\exp(\sqrt{P^{-1}}U\sqrt{P^{-1}})\sqrt{P}$$
\label{prop:Exp_RiemVsMatrix_General}
\end{corollary}
\begin{proof}
Given $P\in\spd_n$ and tangent vector $U\in T_P\spd_n$ identified with the set $S_n$ of symmetric matrices, note that $X\mapsto \sqrt{P^{-1}}X\sqrt{P^{-1}}$ is a symmetry of $\spd_n$ taking $P$ to $I$ and $U$ to $\sqrt{P^{-1}}U\sqrt{P}^{-1}$.  Using the fact that we understand the Riemannian exponential at the basepoint, we see
$\exp_{Riem}(\sqrt{P^{-1}}U\sqrt{P^{-1}})=\exp(\sqrt{P^{-1}}U\sqrt{P^{-1}}).$
It only remains to translate the result back to $P$, giving the claimed formula.
\end{proof}

\begin{proposition}
Let $\log_{\textrm{Riem}}\colon \spd_n\to S_n$ be the Riemannian logarithm map based at $0\in S_n$, and $\log$ be the matrix logarithm (note that while the matrix logaritm is multivalued in general, it is uniquely defined on $S_n$).  Then $\log_{\textrm{Riem}}=\log$.
\label{prop:Log_RiemVsMatrix}
\end{proposition}
\begin{proof}
Defined as the inverse of $\exp_{Riem}$, the Riemannian logarithm must satisfy 
$$\log_{Riem}\circ\exp_{Riem}=\operatorname{id}_{S_n}$$
Let $U\in S_n$ and orthogonally diagonalize as $U=KDK^T$.
Applying the Riemannian exponential, we see $\log_{Riem}(K\exp(D)K^T)=KDK^T$.
Recalling from Lemma \ref{prop:Isometry-TangVector} the relation between isometries of $\spd_n$ and their application on tangent vectors, we see that we may rewrite the left hand side as $\log_{Riem}(K \exp(D) K^T) = K\log_{Riem}(\exp(D))K^T$.
Appropriately cancelling the factors of $K, K^T$ we arrive at the relationship
$$\log_{Riem}(\exp(D))=D.$$
That is, restricted to the diagonal matrices, the Riemannian logarithm is an inverse of the matrix exponential, so Riemannian log equals matrix log.
Re-absorbing the original factors of $K$ shows the same to be true for any positive definite symmetric matrix; thus $\log_{Riem}=\log$.
\end{proof}

As for the exponential, conjugating by a symmetry moving $I$ to an arbitrary point $P$, we may describe the Riemannian logarithm at any point of $\spd_n$.

\begin{corollary}
The logarithm based at an arbitrary point $P\in\spd_n$ is given by $$\log_{Riem,P}(Q)=\sqrt{P}\log(\sqrt{P^{-1}}Q\sqrt{P^{-1}})\sqrt{P}$$
\label{prop:Log_RiemVsMatrix_General}
\end{corollary}

\subsection{Vector-valued Distance}
\label{sec:VecValDist}

Here we collect useful observations about the vector-valued distance on $\spd_n$, culminating in a proof of the fact that it is a complete invariant of pairs of points, as claimed in \S\ref{sec:space-spd}.

\begin{proposition}
The vector-valued distance is well-defined: given any pair $P,Q\in\spd_n$ of points and any two isometries taking $P,Q$ to the basepoint, a diagonal matrix respectively, the diagonal matrices differ at most by a permutation of their entries.
\label{prop:vvd_Defined}
\end{proposition}
\begin{proof}
We see heuristically that there is no remaining continuous degree of freedom by dimension count: the isometry group $GL(n;\RR)$ has dimension $n^2$, and we require $\dim(\spd_n)=n(n+1)/2$ degrees of freedom to translate $P$ to the origin, and a further $\dim O(n)=n(n-1)/2$ degrees of freedom to diagonalize the image of $Q$ while fixing $I$.  As $\dim GL(n;\RR)=\dim \spd_n+\dim O(n)$, there are no remaining continuous degrees of freedom.
To see that the remaining ambiguity is precisely permutation of coordinates, note that conjugating a diagonal matrix by an orthogonal matrix results in another diagonal matrix only if the conjugating matrix is a permutation matrix.
 \end{proof}

\begin{proposition}
If two points $P,Q\in\spd_n$ have the same vector-valued distance from the basepoint $I$, then there is an isometry fixing $I$ taking $P$ to $Q$.
\label{prop:vvd_invariant}
\end{proposition}
\begin{proof}
For two matrices to have the same vector-valued distance from $I$ is equivalent to those two matrices having the same set of eigenvalues.  Let $\lambda_1,\ldots,\lambda_n$ be a list of these eigenvalues with multiplicity, and construct two orthonormal bases $(v_i),(w_i)$ of $\RR^n$ as follows.
For each $i,$ let $v_i$ be an eigenvector of $P$ with eigenvalue $\lambda_i$, and $w_i$ an eigenvector of $Q$ with eigenvalue $\lambda_i$ (in the case the eigenvalues are distinct, such bases are unique up to flipping the sign of each vector, but nontrivial choices must be made in the case of coincident eigenvalues).
Given this pair of orthonormal bases, let $K\in O(n)$ be the orthogonal matrix which takes $(v_i)$ to $(w_i)$.
It is then an easy observation of linear algebra to note that $Q=KPK^{-1}$, but recalling $K^T=K^{-1}$ we see this is interpreted in the geometry of $\spd_n$ to say that there is an isometry $X\mapsto KXK^T$ fixing $I$ and taking $P$ to $Q$.
\end{proof}

Combining Propositions \ref{prop:vvd_Defined} and \ref{prop:vvd_invariant}, after translating appropriately to the basepoint yields the following cornerstone of the theory, showing the vector-valued distance to be the \emph{best possible} invariant.

\begin{corollary}
The vector-valued distance is a complete invariant of pairs of points.  Two pairs of points $(P,Q)$ and $(P^\prime, Q^\prime)$ cam be mapped to one another by an isometry if and only if they have the same vector-valued distance.
\label{prop:vvd_CompleteInvariant}
\end{corollary}

It's important to note that while the vector-valued distance is not \emph{literally} a metric distance (it is vector valued, instead of positive-real-number valued, for one) it enjoys some properties analogous to traditional metric distances.
For a brief review of some of these (the vector-valued triangle inequality, etc) see Kapovich, Leeb \& Porti. \cite{kapovich2017vectorValuedDistance}, and Kapovich, Leeb \& Millison \cite{kapovich2017cvxFns}.

One property distinguishing the vector-valued distance from traditional metrics is its assymmetry.
We will wish to recall this relationship later on, and so prove it here for completeness.

\begin{lemma}
For $P,Q\in \spd_n$, the vector-valued distance satisfies 
$$d_{vv}(P,Q)=-d_{vv}(Q,P)$$
with equality understood up to permutation of coordinates.
\label{prop:vvd_Assymmetry}
\end{lemma}
\begin{proof}
The computation of $d_{vv}(P,Q)$ differs from that of $d_{vv}(Q,P)$ in the first step, where we reduce it to computing a function of the eigenvalues of $P^{-1}Q$ or $Q^{-1}P$ respectively.
Noting these are inverses of one another, their eigenvalues are reciprocals we may perform the following calculation, where $\{\lambda_i(X)\}$ denotes the eigenvalues of $X$.
\begin{equation*}
\begin{split}
    d_{vv}(Q,P)&=\log(\ldots,\lambda_i(Q^{-1}P),\ldots)\\
    &=\log(\ldots,\lambda_i((P^{-1}Q)^{-1}),\ldots)\\
     &=\log(\ldots,\lambda_i((P^{-1}Q)^{-1}),\ldots)\\
     &=\log\left(\ldots,\frac{1}{\lambda_i((P^{-1}Q))},\ldots\right)\\
     &=-\log(\ldots,\lambda_i((P^{-1}Q)),\ldots)\\
     &=-d_{vv}(P,Q)
\end{split}
\end{equation*}
\end{proof}

\subsection{Riemannian Distance}
\label{sec:RiemDist}
This Riemannian metric allows the computation of the length of curves $\gamma\colon[0,1]\to\spd_n$ as
$$\operatorname{length}(\gamma)=\int_0^1\sqrt{\langle \gamma^\prime(t),\gamma^\prime(t)\rangle_{\gamma(t)}}\;dt.$$
This in turn induces a distance function $d: \spd_n \times \spd_n \to \mathbb{R}$, by taking the infimum of the lengths of all paths joining two points:

$$d(P,Q)=\inf_{\begin{smallmatrix}\gamma\colon[0,1]\to\spd_n\\\gamma(0)=P,\;\gamma(1)=Q\end{smallmatrix}}\Big\{\operatorname{length}(\gamma)\Big\}
$$

While for general Riemannian manifolds such a distance function may be impossible to explicitly compute, the symmetries of $\spd_n$ provide a readily computable formula.

\begin{proposition}
The Riemannian distance from the basepoint $I$ to a point $P\in\spd_n$ is given by $d(I,P)=\sqrt{\sum_{i=0}^n \log(\lambda_i(P))}$ where $\{\lambda_i(P)\}$ are the eigenvalues of of $P$.
\label{prop:Distance_Riemannian}
\end{proposition}
\begin{proof}
Let $P\in\spd_n$ be arbitrary, and orthogonally diagonalize as $P=KDK^T$.
As $K\in O(n)$, the isometry $X\mapsto KDK^T$ fixes $I$, so the distance $d(I,P)$ equals the distance $d(I,D)$.  
Note as this action of $K$ is by conjugacy, the diagonal entries $d_i$ of $D$ are precisely the eigenvalues of $P$.
As $D$ lies in the totally geodesic Euclidean subspace $\mathcal{D}$, this distance is realized by the unique Euclidean geodesic connecting $I$ to $D$.
Using Lemma \ref{prop:MaxFlat_Euclidean}, we may translate to familiar coordinates on $\RR^n$ and notice this is the distance from the origin $0$ to the point $x=\left(\log(d_1),\ldots, \log(d_n)\right)$.
That is, $d(I,D)=\sqrt{\sum_{i}\log(d_i)^2}$ as claimed.
\end{proof}

This immediately generalizes to the distance between a pair of arbitrary points, via conjugating by a symmetry moving one to the origin.  
However, with a little more work one may get a simpler expression for the general distance.

\begin{proposition}
The Riemannian distance between two arbitrary points $P,Q\in\spd_n$ is given by $d(P,Q)=\sqrt{\sum_i \log(\lambda_i(P^{-1}Q))}$ where $\{\lambda_i(P^{-1}Q)\}$ are the eigenvalues of of $P^{-1}Q$.
\label{prop:Dist_Riemannian_General}
\end{proposition}
\begin{proof}
If $P,Q$ are arbitrary points in $\spd_n$, we may use an isometry to translate $P$ to the basepoint, while simultaneously moving $Q$ to $R=\sqrt{P^{-1}}Q\sqrt{P^{-1}}$.
As isometries preserve distances, we have $d(P,Q)=d(I,R)$, and by Proposition \ref{prop:Distance_Riemannian}, this distance is completely determined by the eigenvalues of $R$.  
As these are invariant under conjugacy, we replace $R$ with its conjugate by $\sqrt{P^{-1}}$ to get the matrix 
\begin{equation*}
    \begin{split}
        R^\prime&=\sqrt{P}R\sqrt{P}^{-1}\\
        &=\sqrt{P^{-1}}\sqrt{P^{-1}}Q\sqrt{P^{-1}}\sqrt{P}\\
        &=P^{-1}Q
    \end{split}
\end{equation*}
\end{proof}

\subsection{Finsler Distances}
\label{sec:FinslerDist}

The Riemannian distance function on a manifold is completely determined by its Riemannian metric, a choice of inner product on the tangent bundle.
Generalizing this, Finsler metrics are the class of distance functions which may be constructed from  a smoothly varying choice of norm $\|\cdot \|_F$ on the tangent bundle (which need not be induced by an inner product).
The basic theory proceeds in direct analogy to the Riemannian case: the length of a curve $\gamma\colon[0,1]\to\spd_n$ with respect to a Finsler metric is still defined via integration of this norm along the path,and the distance between points by the infimum of this over all rectifiable curves joining them
$$\mathrm{length}_F(\gamma)=\int_0^1\|\gamma'\|_F dt,\hspace{1cm}
 d_F(P,Q)=\inf_{\begin{smallmatrix}\gamma\colon[0,1]\to\spd_n\\\gamma(0)=P,\;\gamma(1)=Q\end{smallmatrix}}\Big\{\operatorname{length}_F(\gamma)\Big\}
$$

The geometry of $\spd_n$ allows the computaiton of all Finsler metrics directly from the vector-valued distance.
As Riemannian metrics are in particular a special case of Finsler metrics, we begin by recasting our previous observations in this light.
In \S\ref{sec:RiemDist} we derived a formula for the Riemannian distance function directly from the infintesimal Riemannian metric.  
But in light of Corollary \ref{prop:vvd_CompleteInvariant}, since the Riemannian distance is a function which depends only on its input points up to isometry, it must also be recoverable from the vector-valued distance.
Indeed, looking at Proposition \ref{prop:Distance_Riemannian} we see there is a simple rephrasing to this effect:

\begin{corollary}
The Riemannian distance from the basepoint $I$ to an arbitrary point $P\in\spd_n$ is the Euclidean norm of the vector-valued distance from $I$ to $P$.
\end{corollary}

One of the great advantages of higher rank symmetric spaces is the generalizations to which this rephrasing lends itself.  
Namely, the Euclidean metric is not special in this construction, and any sufficiently symmetric norm on $\RR^n$ can induce a distance function on $\spd_n$ in this way.

\begin{proposition}
Let $\|-\|$ be any norm on $\RR^n$ which is invariant under the permutation of coordinates. Then $\|-\|$ induces a distance function $d$ on $\spd_n$ by 
$$d(P,Q)=\|d_{vv}(P,Q)\|$$
\end{proposition}
\begin{proof}
We first note this function is well-defined, as by Proposition \ref{prop:vvd_Defined} the vector-valued distance of $(P,Q)$ is well-defined up to permutation of coordinates, and our norm is invariant under this symmetry by hypothesis.
To see that $d$ is in fact a distance function on $\spd_n$, we now need to show it satisfies the axioms of a metric:
\begin{enumerate}
    \item $d(P,Q)\geq 0,\; d(P,Q)=0\implies P=Q$
    \item $d(P,Q)=d(Q,P)$
    \item $d(P,R)\leq d(P,Q)+d(Q,R)$
\end{enumerate}

To check the identity of indescernibles (1), note that $d$ is necessarily nonnegative as $\|-\|$ is, and if $d_(P,Q)=0$ then the norm of $d_{vv}(P,Q)$ is zero, so the vector-valued distance itself is zero.  But as this is a complete invariant and $d_{vv}(P,P)=0$, this means $P=Q$.

Note that symmetry (2) is not automatic as the vector-valued distance itself is asymmetric.  However recalling Lemma \ref{prop:vvd_Assymmetry}, we see that changing the order causes only a global negative sign, and the central symmetry of $\|-\|$, as a virtue of being a norm, gives equality.

The triangle inequality (3) is more subtle, and requires an understanding of the triangle inequality for the vector-valued distance. See the dissertation of Planche \cite{planche1995finslerMetrics}, Chapter 6 and the work of Kapovich, Leeb and Millson \cite{kapovich2017cvxFns} for details.
\end{proof}

For our experiments, the most important such distances are induced by the $\ell_1$ and $\ell_\infty$ norms on $\RR^n$.  For completeness, the resulting distance functions are described below.

\begin{proposition}
The distance function induced from the $\ell_1$ metric applied to the vector-valued distance can be computed  as $d_{F_1}(P,Q)=\sum_{i=1}^n|\log\lambda_i(P^{-1}Q)|$, where $\lambda_i(P^{-1}Q)$ runs over the eigenvalues of $P^{-1}Q$.
\end{proposition}
\begin{proof}
The vector-valued distance $d_{vv}(P,Q)$ is the vector of logarithms of the eigenvalues of $R=P^{-1}Q$, and its $\ell^1$ norm is the sum of their absolute values:
$$\|(\log(\lambda_1(R),\ldots,\lambda_n(R))\|_{\ell^1}=\sum_{i=1}^n|\log\lambda_i(R)|$$
where $\lambda_i(R)$ is the $i^{th}$ eigenvalaue of $R$, in decreasing order.
\end{proof}

A similar calculation yields the formula for the $F^\infty$ distance function.
\begin{proposition}
The distance function induced from the $\ell_\infty$ metric applied to the vector-valued distance can be computed as $d_{F_\infty}(P,Q)=\lambda_1(P^{-1}Q)$ where $\lambda_1(-)$ returns the largest eigenvalue of the input matrix.
\end{proposition}

\subsection{Relations with Other Metrics}
\label{sec:app-relation-with-other-metrics}
Other distances previously used in the literature can be reconstructed from the vector-valued distance, by applying a suitable function:

The \emph{Affine Invariant metric} of \cite{pennec2006spdmagneticResonanceImaging} is nothing but the usual Riemannian metric discussed in \S\ref{sec:RiemDist}.

The  \emph{symmetric Stein divergence} \cite{sra2012steinMetricSPD}, is given by
$$S(P,Q) := \log \det \frac{ P+Q}2 - \frac1 2 \log \det(PQ)$$
This can be computed from the vector-valued distance
$$d_{vv}(P,Q)=\log(\lambda_1(P^{-1}Q),\ldots, \lambda_n(P^{-1}Q))$$
by applying the function
$$\|v\|_S=\sum_{i=1}^n\log\frac{e^{-v_i/2}+e^{v_i/2}}{2}.$$
Indeed 
$$\begin{array}{rl}
S(P,Q) &= \log \det \frac{ P+Q}2 - \frac1 2 \log \det(PQ)\\
&=\log\det P\left(\frac{ \Id+P^{-1}Q}2\right) - \log \det(P\sqrt{P^{-1}Q})\\
&=\log\det \frac{ \Id+P^{-1}Q}2 - \log \det(\sqrt{P^{-1}Q})\\
&=\sum_{i=1}^n\log\lambda_i\left(\frac{ \Id+P^{-1}Q}{2\sqrt{P^{-1}Q}} \right)\\
&=\sum_{i=1}^n\log\left(\frac{ \lambda_i(P^{-1}Q) ^{-1/2}+\lambda_i(P^{-1}Q)^{1/2} }{2} \right)
\end{array}
$$
In particular we obtain, thanks to the vector-valued distance, a more direct proof of \cite{sra2015positiveDefiniteMatrices}. 
%https://arxiv.org/pdf/1110.1773.pdf

Instead the \emph{Log-Euclidean} metric $d_{LE}$ \cite{arsigny2006spdGeometricMeans, arsigny2006logEuclidTensorDiffusion} is flat, and as such doesn't reflect the curved geometry of SPD. 
More precisely $d_{LE}$ is the pushforward, through the exponential map $\exp_{Riem}:S_n\to\spd$ of the Euclidean metric on $S_n$. As a result, for this choice $(\spd_n,d_{LE})$ is \emph{isometric} to the flat manifold $S_n$. 
Since the group $GL(n,\R)$ does not act by isometries on $(\spd_n;d_{LE})$, and the distance is therefore not related to the vector-valued distance nor can be computed from it.

Similarly the \emph{Bures-Wasserstein} metric $d_{BW}$ inspired from quantum information theory \cite{bhatia2019buresWassersteinDistanceSPD} leads to a non-negatively curved manifold, and thus, again, has a different isometry group. More precisely $$d_{BW}(P,Q)=\sqrt{\tr(P)+\tr(Q)-2\sqrt{\tr (PQ)}}.$$
It is computed in \cite[Page 15]{bhatia2019buresWassersteinDistanceSPD}
that the group of isometries of $(\spd_n, d_{BW})$ is reduced to $O(n)$. As a result, once again, $d_{BW}$ cannot be reconstructed from $d_{vv}$.

% \subsection{Transferring Euclidean Functions via Flats}

% As a means of transferring Euclidean notions to $\spd_n$, if $F$ is any function on $\RR^n$ which is invariant under permuting coordinates, then $F$ has a cannonical extension to $\spd_n$ defined by $F(P)=F(d_1,\ldots d_n)$ for $P=KDK^T$.
% This is a powerful means of generalization.
% For example, the Euclidean norm transfers to produce the Riemannian distance function from the basepoint on $\spd_n$, and other norms such as $\ell^1$ and $\ell^\infty$ transfer to Finsler metrics on $\spd_n$.

% \begin{proposition}
% The Riemannian distance function on $\spd_n$ is the transfer of the Euclidean norm $\|v\|=\sqrt{\sum_i v_i^2}$ via orthogonal diagonalization.
% \end{proposition}
% \begin{proof}

% \end{proof}

\section{Gyrocalculus}
\label{sec:appendix-gyrocalc}
A primary difficulty of building analogs of Euclidean quantities in curved spaces is the lack of a vector space structure, making the translation of operations like vector addition or scalar multiplication difficult to immediately interpret.
The need for these is already well-noted stumbling block in hyperbolic geometry, as any algorithm using the Euclidean addition of points cannot be implemented directly (for example considering the Poincare disk model, the sum of two points in the disk need not lie in the disk: and even when it does, the result is rarely geometrically meaningful).
To combat this, means of interfacing with hyperbolic geometry using "vector-space-like" operations was developed by \citet{ungar2008gyrovector}, which provides an analog of addition $\oplus\colon\HH^n\times\HH^n\to\HH^n$
 and of scalar multiplication $\otimes\colon \RR\times\HH^n\to\HH^n$ called `gyro-addition' and 'gyro-scalar multiplication' respectively.
 We give a brief introduction to this general theory below, see Ungar's treatment from the lens of differential geometry \cite{ungar2005DiffGeo} for further information.

\subsection{Gyrogroups}
\label{sec:gyrogroup}
Gyrogroups are a generalization of groups which encode algebraically some of the geometric properties of symmetric spaces.
More precisely, a gyrogroup structure on a set $G$ is given by a binary operation $\oplus$, which is assumed to have an identity element $0\in G$ and left inverses $\ominus g$ for each $g\in G$.
Keeping with the conventions familiar from arithmetic, we write $a\ominus b$ to mean $a\oplus (\ominus b)$.
The crucial difference from group theory is that $\oplus$
 is \emph{not} required to be associative.
Instead, the additional structure of  a \emph{gyration operator} $\gyr\colon G\times G\to\Aut(G)$ 
captures the nonassociativity of $\oplus$ by
$$a\oplus(b\oplus c)=(a\oplus b)\oplus \gyr(a,b)c$$.

For $(G,\oplus,\gyr)$ to form a gyrogroup, an additional axiom is imposed on this gyration, namely that it satisfy
 the \emph{left loop identity}, $\gyr(a,b)=\gyr(a\oplus b,b)$.

Gyrogroups generalize groups in the sense that every group $G$ is a gyrogroup with its usual binary operation as $\oplus$, and trivial gyration.  As with standard groups, it is helpful to have at one's disposal a collection of elementary deductions from these axioms, which may significantly simplify further calculations.

\begin{proposition}
The identity of a gyrogroup is unique, every left inverse is also a right inverse, and every element has a unique (left, and hence also right) inverse.
\label{prop:Gyro_BasicProps_Add}
\end{proposition}

See \cite{ungar2005DiffGeo} \S 3 for a proof of this proposition, which uses only the axioms of a gyrogroup.
It can be shown that when  a gyrogroup structure exists on a set $G$, it is determined by the operation $\oplus$ alone, in the sense that for any $a,b,c$ we have
\begin{equation}
\label{eqn:gyration}
    \gyr(a,b)c=\left(\ominus \left(a\oplus b\right)\right)\oplus\left(a\oplus\left(b\oplus c\right)\right)
\end{equation}

We record also useful properties of the gyration operator following from this, which simplify calculation.

\begin{proposition}
The following gyrations are trivial: the gyration of any element with zero $\gyr(0,a)=\gyr(0,a)=\gyr(\ominus a,a)=\operatorname{id}_G$, or with its inverse $\gyr(\ominus a, a)=\gyr(a,\ominus a)=\operatorname{id}_G$.
A useful consequence of these is the \emph{nested gyration identity}: 
$$\gyr(a,\ominus\gyr(a,b)b)\gyr(a,b)=\operatorname{id}_G$$
\label{prop:Gyro_BasicProps_GyrI}
\end{proposition}
% \begin{proposition}
% Gyration satisfies the identities $\gyr(a,b)0=0$ and $\gyr(a,b)\ominus x=\ominus \gyr(a,b)x$.
% \label{prop:Gyro_BasicProps_GyrII}
% \end{proposition}

These are also proven in \cite{ungar2005DiffGeo} \S 3 , and follow directly from the axioms of a gyrogroup.  

Because of the additional complexity of $\oplus$ compared to the binary operation of a standard group, it is often useful in applications to introduce a second binary operation, the \emph{gyrogroup co-operation} $\boxplus$ and its inverse $\boxminus$, defined by
$$a\boxplus b=a\oplus \gyr(a,\ominus b)b
\hspace{1cm}
a\boxminus b=a\boxplus\ominus a$$
This operation provides a useful shorthand for solving equations in gyrogroups, which we discuss in \ref{sec:gyrogroup-equations}.

Finally, we give a means of computing the VVD in terms of these operations as claimed in \S\ref{sec:gyrocalculus}.

\begin{proposition}
The vector-valued distance from $P$ to $Q$ is the vector of logarithms of the eigenvalues of $(\ominus P)\oplus Q$.
\label{prop:gyro_dist}
\end{proposition}
\begin{proof}
This is the matrix $(\ominus P)\oplus Q=P^{-1}\oplus Q=\sqrt{P^{1}}Q\sqrt{P^{-1}}$, which is conjugate to $P^{-1}Q$ (as in \ref{prop:Dist_Riemannian_General}), and so has the same eigenvalues.
But the logarithm of these eigenvalues is precisely the vector value distance as defined in \S\ref{sec:VVD}.
\end{proof}

\subsection{Gyro-vector Spaces}
\label{sec:Gyro_VecSpace}
Though the operation $\oplus$ is not commutative in the usual sense, a gyrogroup $G$ is called \emph{gyro-commutative} if it commutes \emph{up to gyrations}: ie for every  $a,b\in G,$  $a\oplus b=\gyr(a,b)(b\oplus a)$.
It is within this restricted class of gyro-commutative gyrogroups that a satisfactory analog of familiar vector space operations can be constructed \cite{ungar2018pseudorotations}.

A gyrovector space is a gyro-commutative gyrogrorup $(G,\oplus)$ together with a scalar multiplication $\otimes\colon\RR\times G\to G$ such that $1$ acts as the identity, and its interaction with standard multiplication, gyro-addition and gyration are constrained by
\begin{equation}
    \begin{split}
        r_1\otimes(r_2\otimes a)&=r_1r_2\otimes a\\
        (r_1+r_2)\oplus a&=(r_1\otimes a)\oplus (r_2\otimes a)\\
        r\otimes \gyr(a,b)c&=\gyr(a,b)(r\otimes c)\\
        \gyr(r_1\otimes a,r_2\otimes a)&=I
    \end{split}
\end{equation}

Typically a gyrovector space is also assumed to be constructed within an ambient real inner product space, and there are additional compatibility relations between the operations of $(G,\oplus,\otimes)$ and the ambient vector space addition $(+)$ and norm $\|v\|=\sqrt{v\cdot v}$.

Gyro-vector spaces generalize vector spaces much as gyro-groups generalized groups: every vector space is a gyro-vector space with trivial gyration.
As such, the formalism of gyro-vector spaces provides a convenient generalization where one may
attempt to replace $+,-,\times$ in formulas familiar from Euclidean spaces with $\oplus,\ominus,\otimes$; being careful to recall that 
 gyro-addition is neither commutative nor associative, and gyro-multiplication rarely distributes over $\oplus$.

\subsubsection{Solving Equations in Gyrogroups}
\label{sec:gyrogroup-equations}

 As an example of the difficulties posed by this, if one requires the solution to the Euclidean equation $a+x=b$, it is equally correct to write $x=b-a$ or $x=-a+b$.
 But the translations $x=b\ominus a$ and $x=\ominus a\oplus b$ into a gyrogroup $G$ need not be equal, and generically only the latter solves the gyrovector equation $a\oplus x=b$.
 
To make this more systematic, note that working inwards respecting the order of operations, we are able to solve any equation in a gyrogroup if we compute a \emph{left cancellation law}, \emph{right cancellation law} and \emph{invert scalar multiplication}.

\begin{proposition}[Left-Cancellation]
Let $a,b$ be elements of a gyrogroup $G$.  Then the relation $a\oplus x=b$ is satisfied by the unique value $x=(\ominus a)\oplus b$.
\end{proposition}
\begin{proof}
Substituting the claimed expression for $x$, we verify by direct computation from the axioms of a gyrogroup, and the basic properties of Propositions \ref{prop:Gyro_BasicProps_Add}.
\begin{equation*}
    \begin{split}
        a\oplus x&=a\oplus\left((\ominus a)\oplus b\right)\\
        &=(a\oplus\ominus a)\oplus \gyr(a,\ominus a)b\\
        &=0\oplus\gyr(a,\ominus a)b\\
        &=id_G (b)\\
        &=b
    \end{split}
\end{equation*}
\end{proof}

\begin{proposition}[Right-Cancellation]
Let $a,b$ be elements of a gyrogroup $G$.  Then the relation $x\oplus a=b$ is satisfied by the unique value $x=b\boxminus a=a\ominus \gyr(a,\ominus b)b$, where $\boxminus$ is the additive inverse of the gyrogroup co-operation from Section \ref{sec:gyrogroup}.
\end{proposition}
\begin{proof}
 To begin, we put the proposed solution $b\boxminus a$ in a more usable form:
\begin{equation*}
\begin{split}
    b\boxminus a&=b\boxplus\ominus a\\
    &=b\oplus \gyr(b,\ominus \ominus a)\ominus a\\
    &=b\ominus \gyr(b,a)a\\
\end{split}
\end{equation*}

We now verify the claim by subsituting the given value of $x$, and using the properties described in Propositions \ref{prop:Gyro_BasicProps_Add} and \ref{prop:Gyro_BasicProps_GyrI}, (in particular, in the third step we expand $a$ using nested gyration)
\begin{equation*}
\begin{split}
    x\oplus a&=(b\boxminus a)\oplus a\\
    &=(b\ominus \gyr(b,a)a)\oplus \operatorname{id}_G(a)\\
    &=(b\ominus\gyr(b,a)a)\oplus\left(\gyr(b,\ominus \gyr(b,a)a)\gyr(b,a)a\right)\\
    &=b\oplus(\ominus\gyr(b,a)a\oplus\gyr(b,a)a)\\
    &=b\oplus 0\\
    &=b
\end{split}
\end{equation*}
\end{proof}

\begin{proposition}[Inverting Scalar Multiplication]
Let $r\in\RR$ be any scalar, and $a$ an element of a gyrogroup $G$.  Then the relation $r\otimes x=a$ is satisfied by the unique element $x=\left(\tfrac{1}{r}\right)\otimes a$.
\end{proposition}
\begin{proof}
Substituting $x$ immediately yeidls the result given the axioms of gyro-scalar multiplication:
\begin{equation*}
    \begin{split}
        r\otimes &=r\otimes\left(\frac{1}{r}\otimes a\right)\\
        &=\left(r\times \frac{1}{r}\right)\otimes a\\
        &=1\otimes a\\
        &=a
    \end{split}
\end{equation*}
\end{proof}

% These laws may be phrased operationally: left cancellation allows us to cancel the $a$ in $a\oplus x=b$ to solve for $x$:
% \begin{equation}
%     a\oplus x=b\;\;\implies \;x=(\ominus a)\oplus b.
% \end{equation}
% Right cancellation then allows us to cancel the $a$ in $x\oplus a=b$ to solve for $x:$
% \begin{equation}
%     x\oplus a=b\;\;\implies \;x=b\boxminus a=a\ominus \gyr(a,\ominus b)b
% \end{equation}
% where $\boxminus$ is the additive inverse of the gyrogroup co-operation from Section \ref{sec:gyrogroup}.

% Because scalar multiplication interacts naturally with 'regular' (ie arithematic) multiplication, the associated cancellation law is simply
% \begin{equation}
%   \alpha\otimes x=y\;\;\implies \;x=\left(\frac{1}{\alpha}\right)\otimes y.
% \end{equation}

These three cancellation laws allow one work correctly with the gyro-translations of Euclidean vector space statements.  Take for example the vector space expression $a+rx+b=c$ for vectors $a,b,c,x$ and scalar $r$.  One possible gyro-vector space translation of this is $\left(a\oplus(r\otimes x)\right)\oplus b=c$ --- and given this translation, we may work fully within the gyrovector space to solve for $x$ as follows:

\begin{equation*}
    \begin{split}
        \left(a\oplus(r\otimes x)\right)\oplus b&=c \\
        a\oplus(r\otimes x)&=c\boxminus b\\
        r\otimes x&=(\ominus a)\oplus (c\boxminus b)\\
        x&=\frac{1}{r}\otimes\left((\ominus a)\oplus (c\boxminus b)\right)
    \end{split}
\end{equation*}

\end{document}